\documentclass[11pt]{article}

\usepackage{natbib}
\usepackage{url,fancybox}
\usepackage{amsthm}
\usepackage{lscape}

\usepackage{booktabs}       
\usepackage{amsfonts}       
\usepackage{nicefrac}       
\usepackage{microtype}
\usepackage{graphicx} 
\usepackage{graphics}

\usepackage{algorithm}
\usepackage{algorithmic}

\usepackage{amsmath}
\usepackage{mathrsfs}
\usepackage{amssymb}

\usepackage{color}
\usepackage[left=1in,top=1in,right=1in,bottom=1in,letterpaper]{geometry} 
\usepackage{footnote}
\usepackage[normalem]{ulem}

\usepackage{multirow}
\usepackage{bigstrut}
\usepackage{longtable}
\usepackage{pdfpages}
\usepackage{cleveref}

\numberwithin{equation}{section}
\usepackage{amsmath}
\usepackage{mathrsfs}
\usepackage{amssymb}
\usepackage{amssymb}
\usepackage{graphicx}
\usepackage{subfig}

\usepackage{enumitem}
\usepackage{epstopdf}
\usepackage{color}
\usepackage{multirow}
\usepackage{verbatim}
\usepackage{algorithmic}
\usepackage{bm}

\providecommand{\keywords}[1]
{\small	
  \textbf{\textit{Keywords---}} #1
}
\newtheorem{theorem}{Theorem}[section]

\newtheorem{lemma}[theorem]{Lemma}

\newtheorem{definition}[theorem]{Definition}

\newcommand{\T}{\mathrm{T}}
\def\cM{\mathcal{M}}

\def\bx{\boldsymbol{x}}

\def\M{\mathcal{M}}

\def\bU{\boldsymbol{U}}

\def\b1{\boldsymbol{1}}

\newcommand{\normfro}[1]{\left\| #1 \right\|_{{F}}}

\newcommand{\inp}[2]{\left\langle #1, #2 \right\rangle}

\newcommand{\Prox}{\mathrm{Prox}}

\newcommand{\grad}{\mathop{\textbf{grad}}}
\newcommand{\Retr}{\textrm{Retr}}

\newcommand{\st}{\textrm{s.t.}}
\newcommand{\be}{\begin{equation}}
\newcommand{\ee}{\end{equation}}
\newcommand{\ba}{\begin{array}}
\newcommand{\ea}{\end{array}}

\newcommand{\Tr}{\textrm{Tr}}
\newcommand{\Proj}{\textrm{Proj}}
\newcommand{\argmin}{\mathop{\rm argmin}}
\newcommand{\bS}{\mathbb{S}}
\newcommand{\br}{\mathbb{R}}
\newcommand{\LCal}{\mathcal{L}}
\newcommand{\KNN}{\mathrm{KNN}} 
\newcommand{\ACal}{\mathcal{A}}

\begin{document}

\title{A Manifold Proximal Linear Method for Sparse Spectral Clustering with Application to Single-Cell RNA Sequencing Data Analysis\thanks{The first two authors contributed equally to this paper.}}
\author{Zhongruo Wang\thanks{Department of Mathematics, University of California, Davis}
\and Bingyuan Liu\thanks{Department of Statistics, The Pennsylvania State University}
\and Shixiang Chen\thanks{Department of Industrial \& Systems Engineering, Texas A\&M University}
\and Shiqian Ma\footnotemark[2]
\and Lingzhou Xue\footnotemark[3]
\and Hongyu Zhao\thanks{Department of Biostatistics, Yale University}
}
\date{\today}
\maketitle

\begin{abstract} 
Spectral clustering is one of the fundamental unsupervised learning methods and is widely used in data analysis. Sparse spectral clustering (SSC) imposes sparsity to the spectral clustering and it improves the interpretability of the model. This paper considers a widely adopted model for SSC, which can be formulated as an optimization problem over the Stiefel manifold with nonsmooth and nonconvex objective. Such an optimization problem is very challenging to solve. Existing methods usually solve its convex relaxation or need to smooth its nonsmooth part using certain smoothing techniques. In this paper, we propose a manifold proximal linear method (ManPL) that solves the original SSC formulation. We also extend the algorithm to solve the multiple-kernel SSC problems, for which an alternating ManPL algorithm is proposed. Convergence and iteration complexity results of the proposed methods are established. We demonstrate the advantage of our proposed methods over existing methods via the single-cell RNA sequencing data analysis.
\end{abstract}


\maketitle

\section{Introduction}

Clustering is a fundamental unsupervised learning problem with wide applications. The hierarchical clustering, $K$-means clustering and spectral clustering (SC) methods are widely used in practice \citep{friedman2001elements}. It is known that interpretation of the dendrogram in hierarchical clustering can be difficult in practice, especially for large datasets. The $K$-means clustering, closely related to Lloyd's algorithm, does not guarantee to find the optimal solution and performs poorly for non-linearly separable or non-convex clusters. SC \citep{chung1997spectral,shi2000normalized,ng2002spectral} is a graph-based clustering method and it provides a promising alternative for identifying locally connected clusters.

Given the data matrix \(X=[\bx_{1}, \ldots, \bx_{n}] \in \mathbb{R}^{p \times n}\), where \(n\) is the number of data points and \(p\) is the feature dimension, SC constructs a symmetric affinity matrix \(S=\left(s_{i j}\right)_{n \times n}\), where \(s_{i j} \geq 0\) measures the pairwise similarity between two data samples \(\bx_{i}\) and \(\bx_{j}\) for \(i,j=1,\ldots,n\). 
Denote diagonal matrix \(D=\operatorname{Diag}\left(d_{1}, \ldots, d_{n}\right)\) with \(d_{i}=\sum_{j=1}^{n} s_{i j}\). The main step of SC is to compute the following eigenvalue decomposition:
\be\label{SC}
\min_{U \in \mathbb{R}^{n \times C}} \langle  U  U^{\top}, L \rangle, \ \st, \ U^{\top} U=I_C,
\ee
where \(L=I_n - D^{-1/2} S D^{-1 / 2}\) is the normalized Laplacian matrix, $I_C$ denotes the $C\times C$ identity matrix, and $C$ is the number of clusters. The rows of \(U\) can be regarded as an embedding of the data \(X\) from $\mathbb{R}^p$ to $\mathbb{R}^C$. The cluster assignment is then decided after using a standard clustering method such as the $K$-means clustering on the estimated embedding matrix \(\hat{U}\) obtained by solving \eqref{SC}. Ideally, $\hat{U}$ should be a sparse matrix such that $\hat{U}_{ij}\neq 0$ if and only if the sample $i$ belongs to the $j$-th cluster. Therefore $\hat{U}\hat{U}^\top$ should be a block diagonal matrix which is also sparse. To this end, the sparse spectral clustering (SSC) \citep{lu2016convex,lu2018nonconvex,park2018spectral} is proposed to impose sparsity on $\hat{U}\hat{U}^\top$, which leads to the following optimization problem:
\be\label{SSC}
\min_{U \in \mathbb{R}^{n \times C}} \langle  U  U^{\top}, L \rangle + \lambda\|UU^\top\|_1, \ \st, \ U^{\top} U=I_C,
\ee
where $\|Z\|_1 = \sum_{ij}|Z_{ij}|$ is the entry-wise $\ell_1$ norm of $Z$ and it promotes the sparsity of $Z$, and $\lambda>0$ is a weighting parameter.

In practice, the performance of SSC is sensitive to a single measure of similarity between data points, and there are no clear criteria to choose an optimal similarity measure. Moreover, for some very complex data such as the single-cell RNA sequencing (scRNA-seq) data \citep{kiselev2019challenges}, one may benefit from considering multiple similarity matrices because they provide more information to the data. The next-generation sequencing technologies provide large detailed catalogs of the transcriptomes of massive cells to identify putative cell types. Clustering high-dimensional scRNA-seq data provides an informative step to disentangle the complex relationship between different cell types. {For example, it is important to characterize the patterns of  monoallelic gene expression across mammalian cell types \citep{deng2014single}, explore 
the mechanisms that control the progression of lung progenitors across distinct  cell types \citep{treutlein2014reconstructing}, or study the functionally distinct lineage in the bone marrow across mouse conventional dendritic cell types \citep{schlitzer2015identification}.} To this end, \citet{park2018spectral} suggest the following similarity matrices which lead to multiple-kernel SSC (MKSSC): 
\[K_{\delta,m}(i,j)=\exp\left(\frac{\|\bx_i-\bx_j\|^2}{2\epsilon_{ij}^2}\right), \epsilon_{ij}=\frac{\delta(\mu_i+\mu_j)}{2}, \mu_i=\frac{\sum_{\ell\in\KNN (i)}\|\bx_i-\bx_\ell\|}{m},\]
where $\KNN(i)$ represents a set of sample indices that are the top $m$ nearest neighbors of the sample $\bx_i$. The parameters $\delta$ and $m$ control the width of the neighborhoods. We use $S(\delta)$ and $S(m)$ to denote the sets of possible choices of $\delta$ and $m$, respectively. Then the total number of similarity matrices is equal to $T = |S(\delta)|\cdot |S(m)|$. We denote the normalized Laplacian matrices corresponding to these $T$ similarity matrices as $L^{(\ell)}$, $\ell=1,\ldots,T$. The MKSSC can be formulated as the following optimization problem:
\begin{align}\label{prob:mkssc}
\min_{U\in\br^{n\times C},w\in\br^T} & \quad \bar{F}(U,w)\equiv \left\langle UU^\top, \sum_{\ell=1}^T w_\ell L^{(\ell)} \right\rangle + \lambda\|UU^\top\|_1 + \rho\sum_{\ell=1}^T w_\ell\log(w_\ell)\\
\st        & \quad U^{\top} U=I_C, \sum_{\ell=1}^T w_\ell =1, w_\ell \geq 0, \ell=1,\ldots,T, \nonumber  
\end{align}
where $w_\ell, \ell=1,\ldots,T$ are unknown weightings of the kernels, and $\rho\sum_{\ell=1}^T w_\ell\log(w_\ell)$ serves as an entropy regularization term,  {and $\lambda,\rho$ are two regularization parameters. 

Note that both SSC \eqref{SSC} and MKSSC \eqref{prob:mkssc} are nonconvex and nonsmooth with Riemannian manfiold constraints. Therefore, they are both numerically challenging to solve. In this paper, we propose a manifold proximal linear (ManPL) method for solving the SSC \eqref{SSC} {in Section \ref{sec:ManPL-SSC}}, and an alternating ManPL (AManPL) method for solving the MKSSC \eqref{prob:mkssc} {in Section \ref{sec:MKSSC}}. {The convergence rates of ManPL and AManPL are rigorously analyzed. Moreover, we present numerical results in Section \ref{simulation} to demonstrate the advantages of ManPL and AManPL in benchmark datasets, synthetic datasets, and real datasets in scRNA-seq data analysis.}

{\bf Our contributions} lie in several folds.
\begin{enumerate}
\item[(i)] We propose ManPL method for solving SSC \eqref{SSC} and AManPL method for solving MKSSC \eqref{prob:mkssc}.
\item[(ii)] We analyze the convergence and iteration complexity of both ManPL and AManPL.
\item[(iii)] We propose a proximal point algorithm based on a semi-smooth Newton method for solving the convex subproblems arising from ManPL and AManPL. 
\item[(iv)] We apply our proposed methods to clustering of scRNA-seq data.
\end{enumerate}

{\bf Notation.} Throughout this paper, we use $\M$ to denote the Stiefel manifold. The smoothness, convexity, and Lipschitz continuity of a function $f$ are always interpreted as the function is considered in the ambient Euclidean space. We use $\bS_{+}^n$ to denote the set of $n\times n$ positive semidefinite matrices, and $\Tr(Z)$ to denote the trace of matrix $Z$. 

\section{A Manifold Proximal Linear Method for SSC}\label{sec:ManPL-SSC}

Since SSC \eqref{SSC} is both nonsmooth and nonconvex, it is numerically challenging to solve. In the literature, convex relaxations and smooth approximations of \eqref{SSC} have been suggested. In particular, \citet{lu2016convex} proposed to replace $UU^\top$ with a positive semidefinite matrix $P$ and solve the following convex relaxation:
\be\label{SSC-convex-relax-Lu}
\min_{P\in\bS_{+}^n} \ \langle P, L\rangle + \lambda\|P\|_1, \ \st, \ 0\preceq P\preceq I, \ \Tr(P)=C.
\ee
This convex problem \eqref{SSC-convex-relax-Lu} can be solved by classical optimization algorithms such as ADMM. Denote the solution of \eqref{SSC-convex-relax-Lu} by $\hat{P}$, the solution of \eqref{SSC} can be approximated by the top $C$ eigenvectors of $\hat{P}$. In another work, 
\citet{lu2018nonconvex} proposed a nonconvex ADMM to solve the following smooth variant of \eqref{SSC}:
\be\label{SSC-nonconvex-Lu}
\min_{U\in\br^{n\times C},P\in\bS_+^n} \ \langle UU^\top, L\rangle + g_\sigma(P), \ \st, \ P=UU^\top, U^\top U = I_C,
\ee
where $g_\sigma(\cdot)$ is a smooth function with smoothing parameter $\sigma>0$ that approximates the $\ell_1$ regularizer $\lambda\|\cdot\|_1$. In \citep{lu2018nonconvex}, the authors used the following smooth function:
\be\label{def-g}
g_\sigma(P) := \max_Z \ \langle P,Z\rangle - \frac{\sigma}{2}\|Z\|_F^2, \ \st, \ \|Z\|_\infty\leq \lambda,
\ee
where $\|Z\|_\infty = \max_{ij}|Z_{ij}|$. The nonconvex ADMM for solving \eqref{SSC-nonconvex-Lu} typically iterates as
\begin{subequations}\label{eq:non-convex:admm}
\begin{align}
U^{k+1} & := \argmin_{U\in\br^{n\times C}} \ \LCal(U,P^k;\Lambda^k), \ \st, \ U^\top U = I_C, \label{eq:non-convex:admm-1}\\
P^{k+1} & := \argmin_{P\in\bS_+^n} \ \LCal(U^{k+1},P;\Lambda^k), \label{eq:non-convex:admm-2} \\
\Lambda^{k+1} & := \Lambda^k - \mu(P^{k+1}-U^{k+1}{U^{k+1}}^\top),\label{eq:non-convex:admm-3}
\end{align}
\end{subequations}
where the augmented Lagrangian function $\LCal$ is defined as
\[\LCal(U,P;\Lambda) := \langle UU^\top, L\rangle + g_\sigma(P) - \langle\Lambda,P-UU^\top \rangle + \frac{\mu}{2}\|P-UU^\top\|_F^2,\]
and $\mu>0$ is a penalty parameter. The two subproblems \eqref{eq:non-convex:admm-1} and \eqref{eq:non-convex:admm-2} are both relatively easy to solve. The reason to use the smooth function $g_\sigma(\cdot)$ to approximate $\lambda\|\cdot\|_1$ in \eqref{SSC-nonconvex-Lu} is for the purpose of convergence guarantee. In \citep{lu2018nonconvex}, the authors proved that any limit point of the sequence generated by the nonconvex ADMM \eqref{eq:non-convex:admm} is a stationary point of \eqref{SSC-nonconvex-Lu}. This result relies on the fact that function $g_\sigma$ is smooth. If one applies ADMM to the original SSC \eqref{SSC}, then no convergence guarantee is known.

Note that both the convex relaxation \eqref{SSC-convex-relax-Lu} and the smooth approximation \eqref{SSC-nonconvex-Lu} are only approximations to the original SSC \eqref{SSC}. In this section, we introduce our ManPL algorithm that solves the original SSC \eqref{SSC} directly. For the ease of presentation, we rewrite \eqref{SSC} as
\be\label{SSC-manopt-rewrite}
\min_U \ F(U) \equiv f(U) + h(c(U)), \ \st, \ U\in\M,
\ee
where $f(U) = \langle UU^\top, L\rangle$, $h(\cdot) = \lambda\|\cdot\|_1$, $c(U) = UU^\top$, $\M = \{U\in\br^{n\times C}\mid U^\top U=I_C\}$ is the Stiefel manifold. Moreover, note that $f$ and $c$ are smooth mappings, and $h$ is nonsmooth but convex in the ambient Euclidean space. Therefore, \eqref{SSC-manopt-rewrite} is a Riemannian optimization problem with nonsmooth and nonconvex objective function. Furthermore, throughout this paper, we use $L_f$, $L_c$, $L_h$ to denote the Lipschitz constants of $\nabla f$, $\nabla c$, and $h$, respectively.  Riemannian optimization has drawn much attention recently, due to its wide applications, including low rank matrix completion \citep{boumal2011rtrmc}, phase retrieval \citep{Boumal-phase-retrieval-2018,Sun-Ju-geometric-phase-retrieval-2018}, phase synchronization \citep{Boumal-phase-synchronization-2016,Liu-generalized-power-phase-synchronization-2017}, and dictionary learning \citep{Sra-Riemannian-dictionary-learning-2016,Sun-CDR-part1-2017}.  Several important classes of algorithms for Riemannian optimization with a smooth objective function were covered in the monograph \citep{AbsMahSep2008}. On the other hand, there has been very limited number of algorithms for Riemannian optimization with nonsmooth objective until very recently. The most natural idea for this class of optimization problems is the Riemannian subgradient method (RSGM) \citep{ferreira1998subgradient,grohs2016varepsilon,Hosseini-Uschmajew-2017}. Recently, \citet{li2019nonsmooth} studied the RSGM for Riemannian optimization with weakly convex objective. In particular, they showed that the number of iterations needed by RSGM for obtaining an $\epsilon$-stationary point is $O(\epsilon^{-4})$. Motivated by the proximal gradient method for solving composite minimization in Euclidean space, \citet{chen2018proximal} proposed a manifold proximal gradient method (ManPG) for solving the following Riemannian optimization problem: 
\begin{equation}\label{prob:manpg}
\min_X \ f(X) + h(X), \ \st, \ X\in\mathcal{M},
\end{equation}
where $\M$ is the Stiefel manifold, $f$ is a smooth function, and $h$ is a nonsmooth and convex function. A typical iteration of ManPG for solving \eqref{prob:manpg} is:
\begin{subequations}
\begin{align}
V^{k} & := \argmin_V \ \langle \nabla f(X^k), V \rangle + h(X^k+ V) + \frac{1}{2t}\|V\|_F^2, \ \st, \ V\in\T_{X^k}\M, \label{ManPG-alg-1}\\
X^{k+1} & := \Retr_{X^k}(\alpha_kV^{k}),\label{ManPG-alg-2}
\end{align}
\end{subequations}
where $\T_U\M$ denotes the tangent space of $\M$ at $U$, $\alpha_k>0$ is a step size, and $\Retr$ denotes the retraction operation. For the Stiefel manifold, $\T_U\M := \{V\in\br^{n\times C}\mid V^\top U + U^\top V = 0\}$.
Comparing with \eqref{prob:manpg}, we note that  \eqref{SSC-manopt-rewrite} is more difficult to solve, because of the nonconvex term $c(U)$. In fact, ManPG cannot be used to solve the SSC \eqref{SSC} because of the existence of the nonconvex term $UU^\top$ composite with the $\ell_1$ norm. As a result, new algorithm is demanded for solving SSC \eqref{SSC}. The iteration complexity of ManPG is proved to be $O(\epsilon^{-2})$ for obtaining an $\epsilon$-stationary point of \eqref{prob:manpg} \citep{chen2018proximal}, which is better than the complexity of RSGM \citep{li2019nonsmooth}. Variants of ManPG have been designed for different applications, such as alternating ManPG for sparse PCA and sparse CCA \citep{Chen-AManPG-2019}, FISTA for sparse PCA \citep{HW2019a}, manifold proximal point algorithm for robust subspace recovery and orthogonal dictionary learning \citep{chen2019manifold-Asilomar,chen2019manifold}, and stochastic ManPG \citep{Wang-stochastic-ManPG} for online sparse PCA. Moreover, ManPG has been extended to more general Riemannian proximal gradient method \citep{HW2019b}. Motivated by the success of ManPG and its variants, we propose a manifold proximal linear algorithm for solving SSC \eqref{SSC-manopt-rewrite}.

The proximal linear method has recently drawn great research attentions. It targets to solve the optimization problem in the form of \eqref{SSC-manopt-rewrite} without the manifold constraint, i.e., 
\be\label{PLM-prob}
\min_{x\in\br^n} \ f(x) + h(c(x)), 
\ee
where $f: \br^n\to\br$ and $c:\br^n\to\br^m$ are smooth mappings, $h:\br^m\to\br$ is convex and nonsmooth. The proximal linear method for solving \eqref{PLM-prob} iterates as follows:
\be\label{PLM-iter}
x^{k+1} := \argmin_x \ \langle \nabla f(x^k), x-x^k \rangle + h(c(x^k) + J(x^k) (x-x^k)) + \frac{1}{2t}\|x-x^k\|_2^2,
\ee
where $J(x) = \nabla c(x)$ is the Jacobian of $c$, and $t>0$ is a step size. Note that since $h$ is convex, the update \eqref{PLM-iter} is a convex problem. This method has been studied recently by \citep{lewis2016proximal,drusvyatskiy2018efficiency,Duchi-Ruan-SIOPT-2018} and applied to solving many important applications such as robust phase retrieval \citep{duchi2017solving}, robust matrix recovery \citep{Charisopoulos-prox-linear-MC-2019}, and robust blind deconvolution \citep{charisopoulos2019composite}. 

Due to the nonconvex constraint $U\in\M$, solving \eqref{SSC-manopt-rewrite} is more difficult than \eqref{PLM-prob}. Motivated by ManPG and the proximal linear method \eqref{PLM-iter}, we propose a ManPL algorithm for solving \eqref{SSC-manopt-rewrite}. A typical iteration of the ManPL algorithm for solving \eqref{SSC-manopt-rewrite} is:
\begin{subequations}\label{ManPL-alg}
\begin{align}
V^{k} & := \argmin_V \ \langle \nabla f(U^k), V \rangle + h(c(U^k) + J(U^k) V) + \frac{1}{2t}\|V\|_F^2, \ \st, \ V\in\T_{U^k}\M, \label{ManPL-alg-1}\\
U^{k+1} & := \Retr_{U^k}(\alpha_kV^{k}).\label{ManPL-alg-2}
\end{align}
\end{subequations}
Similar to \eqref{ManPG-alg-1}, the equation \eqref{ManPL-alg-1} computes the descent direction $V$ by minimizing a convex function over the tangent space of $\M$. However, solving \eqref{ManPL-alg-1} is more difficult than \eqref{ManPG-alg-1} because of the non-trivial affine function, i.e., $c(U^k) + J(U^k) V$, composite with the nonsmooth function $h$. Moreover, the difference of \eqref{ManPL-alg-1} and \eqref{PLM-iter} is the constraint in \eqref{ManPL-alg-1}, which is needed in the Riemannian optimization setting. Fortunately, \eqref{ManPL-alg-1} can still be solved efficiently by a proximal point algorithm combined with a semi-smooth Newton method, which will be elaborated in Section \ref{sec:ssn}. The retraction step \eqref{ManPL-alg-2} brings the iterate back to the manifold $\M$.

The complete description of the ManPL for solving SSC \eqref{SSC-manopt-rewrite} is given in Algorithm \ref{alg:manpl-ssc}. The step \eqref{eq:line-search} is a line search step to find the step size $\alpha_k$ such that there is a sufficient decrease on the function $F$. 

\begin{algorithm}[ht]
\begin{algorithmic}
\STATE{Input: initial point $U^0\in\M$, parameters $\gamma \in (0,1)$, $t>0$}
\FOR{$k = 0, 1, \ldots$}
    \STATE{Calculate $V^{k}$ by solving \eqref{ManPL-alg-1}} 
    \STATE{Let $j_k$ be the smallest nonnegative integer such that}
        \be\label{eq:line-search}F(\Retr_{U^k}(\gamma^{j_k}V^{k})) \leq F(U^k) - \frac{\gamma^{j_k}}{2t}\|V^{k}\|_F^2\ee
    \STATE{Let $\alpha_k = \gamma^{j_k}$ and compute $U^{k+1}$ by \eqref{ManPL-alg-2}}
\ENDFOR 
\end{algorithmic}
\caption{The ManPL for SSC \eqref{SSC-manopt-rewrite}}\label{alg:manpl-ssc}
\end{algorithm}

The main convergence and iteration complexity result of ManPL (Algorithm \ref{alg:manpl-ssc}) is given in Theorem \ref{thm:manpl-convergence}. Its proof is given in the appendix. 
\begin{theorem}\label{thm:manpl-convergence}
Assume $F(U)$ is lower bounded by $F^*$.	The limit point of the sequence $\{U^{k}\}$ generated by ManPL (Algorithm \ref{alg:manpl-ssc}) is a stationary point of \eqref{SSC-manopt-rewrite}. Moreover, ManPL returns an $\epsilon$-stationary point of \eqref{SSC-manopt-rewrite} in $O(\epsilon^{-2})$ iterations. 
\end{theorem}

\section{A Semi-Smooth Newton-based Proximal Point Algorithm for the Subproblem}\label{sec:ssn}

In this section, we introduce a proximal point algorithm (PPA) combined with a semi-smooth Newton method (SSN) for solving the subproblem \eqref{ManPL-alg-1} in ManPL. The notion of semi-smoothness was originally introduced by  \citep{mifflin1977semismooth} for real valued functions and later extended to vector-valued mappings by \citep{qi1993nonsmooth}. The SSN method has recently received significant amount of attention due to
its success in solving structured convex problems to a high accuracy in problems such as LASSO \citep{Li-Sun-Toh-Lasso,yang2013proximal}, convex clustering \citep{wang2010solving}, SDP \citep{zhao2010newton}, and convex composite problems \citep{xiao2018regularized}.

For simplicity of the notation, we omit the index $k$ in \eqref{ManPL-alg-1}, and denote $Z_1 := \nabla f(U^k)$, $Z_2:=c(U^k)$, $J = J(U^k)$, and operator $\ACal: V \to V^\top U^k + (U^k)^\top V$. Therefore, \eqref{ManPL-alg-1} reduces to the following form: 
\begin{equation}\label{newton_ppa}
\min_{V,Y} \ \frac{1}{2t} \|V + Z_1 \|_F^2 + h(Z_2 + Y), \ \st, \ \ACal(V) = 0, \ Y  = JV.  
\end{equation}
Note that we have introduced a variable $Y$ to replace $JV$. The Lagrangian function for \eqref{newton_ppa} is given by: 
\[
\mathcal{L}(V,Y;\Gamma_1,\Gamma_2) = \frac{1}{2t} \|V + Z_1 \|_F^2 + h(Z_2 + Y) - \langle \Gamma_1, \mathcal{A}(V)\rangle - \langle \Gamma_2, JV - Y\rangle,
\]
where $\Gamma_1$ and $\Gamma_2$ are the Lagrange multipliers associated to the two equality constraints. Therefore, \eqref{newton_ppa} is equivalent to
\begin{equation}\label{lagrangian}
	\min_{V,Y} \{G(V,Y) := \max_{\Gamma_1,\Gamma_2} \mathcal{L}(V,Y;\Gamma_1,\Gamma_2)\}.
\end{equation} 
The minimization problem \eqref{lagrangian} can be solved by a PPA, which iterates as:
\begin{align}
	(V^{k+1},Y^{k+1}) & := \argmin_{V,Y} G(V^k,Y^k) + \frac{1}{2\beta}\left(\|V - V^k\|_F^2 + \|Y-Y^k\|_F^2 \right)	\nonumber \\
	& := \argmin_{V,Y}\max_{\Gamma_1,\Gamma_2} \mathcal{L}(V,Y;\Gamma_1,\Gamma_2) + \frac{1}{2\beta}\left(\|V - V^k\|_F^2 + \|Y-Y^k\|_F^2 \right)	, \label{primal_problem_PPA}
\end{align}
where $\beta>0$ is a parameter. The problem \eqref{primal_problem_PPA} is equivalent to:
\begin{equation}\label{Saddle point formulation}
\max_{\Gamma_1,\Gamma_2}\min_{V,Y} \mathcal{L}(V,Y;\Gamma_1,\Gamma_2) + \frac{1}{2\beta}\left(\|V - V^k\|_F^2 + \|Y-Y^k\|_F^2 \right)	.
\end{equation}
Note that the minimization part of \eqref{Saddle point formulation} is strongly convex and admits a closed-form solution given by: 
\begin{equation}\label{Saddle point formulation-min-sol}
V = \frac{t\beta}{t+\beta}\left(\ACal^*(\Gamma_1)+J^\top\Gamma_2-\frac{1}{t}Z_1+\frac{1}{\beta}V^k\right), \ Y = \Prox_{\beta h}(Z_2+Y^k-\beta\Gamma_2)-Z_2,
\end{equation}
where $\Prox_g$ denotes the proximal mapping of function $g$, which is defined as:
\[
\Prox_g(Z) := \argmin_X \ g(X) + \frac{1}{2}\|X-Z\|_F^2.
\]
For simplicity of the notation, we define function $E(\Gamma_2) := Z_2 + Y^k - \beta \Gamma_2$.  
Substituting \eqref{Saddle point formulation-min-sol} to \eqref{Saddle point formulation}, and using the Moreau identity, we know that \eqref{Saddle point formulation} is equivalent to: 
\begin{align}\label{Saddle point formulation-equiv}
\max_{\Gamma_1,\Gamma_2} \Theta(\Gamma_1,\Gamma_2):= &  -\frac{1}{2}\frac{t\beta}{t+\beta}\left\|\frac{1}{t}Z_1-\frac{1}{\beta}V^k-\ACal^*(\Gamma_1)-J^\top\Gamma_2\right\|_F^2 \\ 
& + h(\Prox_{\beta h}E(\Gamma_2)) + \beta\|\Prox_{h^*/\beta}(E(\Gamma_2)/\beta)\|_F^2 +\langle \Gamma_2,Y^k\rangle - \frac{\beta}{2}\|\Gamma_2\|_F^2, \nonumber
\end{align}
where $h^*$ denotes the conjugate function of $h$. Now by denoting 
\begin{equation}\label{def-Psi}
\Psi(\Gamma_2) := \max_{\Gamma_1} \Theta(\Gamma_1,\Gamma_2), 
\end{equation}
it is easy to verify that $\Psi(\Gamma_2)$ is strongly concave and continuously differentiable \citep{Li-Sun-Toh-Lasso}, and its unique maximizer is found by solving the following nonsmooth system:
\begin{equation}\label{SSN-equation}
\nabla \Psi(\Gamma_2) = 0.
\end{equation}
Solving \eqref{SSN-equation} can be done by using SSN \citep{Li-Sun-Toh-Lasso,xiao2018regularized}. After we obtain the solution to \eqref{SSN-equation}, the optimal $\Gamma_1$ can be found by solving the maximization problem in \eqref{def-Psi}, which is an easy least-squares problem. 

To summarize, the PPA for solving \eqref{ManPL-alg-1} is given by  \eqref{primal_problem_PPA}, and its solution is given by \eqref{Saddle point formulation-min-sol}. The required $\Gamma_2$ in \eqref{Saddle point formulation-min-sol} is obtained by solving \eqref{SSN-equation} using SSN, and $\Gamma_1$ in \eqref{Saddle point formulation-min-sol} is obtained by solving the least-squares problem in \eqref{def-Psi}. The convergence of the PPA and the SSN has been well studied in the literature \citep{Li-Sun-Toh-Lasso,yang2013proximal}.

\section{An Alternating ManPL Method for Multiple-Kernel SSC}\label{sec:MKSSC}

In this section, we consider the multiple-kernel SSC \eqref{prob:mkssc}. 
\citet{park2018spectral} consider to solve the following relaxation of \eqref{prob:mkssc} by letting $P=UU^\top$:
\begin{align}\label{prob:mkssc-park}
\min_{P,w} & \quad \left\langle P, \sum_{\ell=1}^T w_\ell L^{(\ell)} \right\rangle + \lambda\|P\|_1 + \rho\sum_{\ell=1}^T w_\ell\log(w_\ell)\\
\st        & \quad \Tr(P)=C, 0\preceq P\preceq I, \sum_{\ell=1}^T w_\ell =1, w_\ell \geq 0, \ell=1,\ldots,T. \nonumber  
\end{align}
Note that this is still a nonconvex problem due to the bilinear term in the objective function. 
\citet{park2018spectral} suggested to use an alternating minimization algorithm (AMA) to solve \eqref{prob:mkssc-park}. Note that this method is named MPSSC in \citep{park2018spectral}. In the $k$-th iteration of AMA, one first fixes $w$ as $w^k$ and solves the resulting problem with respect to $P$ to obtain $P^{k+1}$, and then fixes $P$ as $P^{k+1}$ and solves the resulting problem with respect to $w$ to obtain $w^{k+1}$. In particular, when $w$ is fixed as $w^k$, problem \eqref{prob:mkssc-park} reduces to
\begin{equation}\label{prob:mkssc-park-P}
\min_{P} \quad \left\langle P, \sum_{\ell=1}^T w_\ell^k L^{(\ell)} \right\rangle + \lambda\|P\|_1, \ \st, \ \Tr(P)=C, 0\preceq P\preceq I,
\end{equation}
which is a convex problem and can be solved via convex ADMM algorithm. When $P$ is fixed as $P^{k+1}$, problem \eqref{prob:mkssc-park} reduces to 
\begin{equation}\label{prob:mkssc-park-w}
\min_{w} \quad c^\top w + \rho\sum_{\ell=1}^T w_\ell\log(w_\ell), \ \st, \ \sum_{\ell=1}^T w_\ell =1, w_\ell \geq 0, \ell=1,\ldots,T, 
\end{equation}
where $c_\ell = \langle P^{k+1}, L^{(\ell)}\rangle$, $\ell=1,\ldots,T$. This is also a convex problem and it can be easily verified that \eqref{prob:mkssc-park-w} admits a closed-form solution given by
\be\label{prob:mkssc-w-sol}
w_\ell = \frac{\exp(-c_\ell/\rho)}{\sum_{j=1}^T \exp(-c_j/\rho)}, \ \ell = 1,\ldots, T.
\ee
In summary, a typical iteration of the AMA algorithm proposed by  \citep{park2018spectral} is as follows:
\be\label{AMA-park-zhao}
\left\{\begin{array}{l}
\mbox{update } P^{k+1} \mbox{ by solving \eqref{prob:mkssc-park-P}}  \\
\mbox{update } w^{k+1} \mbox{ by solving \eqref{prob:mkssc-park-w}}.
\end{array}\right.
\ee

Another approach to approximate \eqref{prob:mkssc} is to combine the idea of AMA \eqref{AMA-park-zhao} and the nonconvex ADMM for solving the smooth problem \eqref{SSC-nonconvex-Lu}. In particular, one can solve the following smooth variant of \eqref{prob:mkssc}:
\begin{equation}\label{prob:mkssc-smooth}
\begin{split}
\min_{U\in\br^{n\times C},w\in\br^T} & \quad  \left\langle UU^\top, \sum_{\ell=1}^T w_\ell L^{(\ell)} \right\rangle + g_\sigma(UU^\top) + \rho\sum_{\ell=1}^T w_\ell\log(w_\ell)\\
\st        & \quad U^{\top} U=I_C, \sum_{\ell=1}^T w_\ell =1, w_\ell \geq 0, \ell=1,\ldots,T, 
\end{split}
\end{equation}
where $g_\sigma(\cdot)$ is the smooth approximation to $\lambda\|\cdot\|_1$ defined in \eqref{def-g}. When fixing $w$, \eqref{prob:mkssc-smooth} is in the same form as the smoothed SSC \eqref{SSC-nonconvex-Lu}, so it can be solved by the nonconvex ADMM \eqref{eq:non-convex:admm}. When fixing $U$, \eqref{prob:mkssc-smooth} 
is in the same form as \eqref{prob:mkssc-park-w}, and admits a closed-form solution \eqref{prob:mkssc-w-sol}. In summary, the AMA+ADMM algorithm for solving \eqref{prob:mkssc-smooth} works as follows:
\be\label{prob:mkssc-smooth-AMA+NADMM}
\left\{\ba{l}
\mbox{ update } U^{k+1} \mbox{ by solving \eqref{prob:mkssc-smooth} with } w \mbox{ fixed as } w^k \mbox{ using nonconvex ADMM \eqref{eq:non-convex:admm}} \\
\mbox{ update } w^{k+1} \mbox{ by solving \eqref{prob:mkssc-smooth} with } U \mbox{ fixed as } U^{k+1} \mbox{ using \eqref{prob:mkssc-w-sol}}.
\ea\right.
\ee

By exploiting the structure of \eqref{prob:mkssc}, we propose to solve \eqref{prob:mkssc} by an alternating ManPL algorithm (AManPL). More specifically, in the $k$-th iteration of AManPL, we first fix $w$ as $w^k$, then \eqref{prob:mkssc} reduces to
\be\label{prob:mkssc-U}
\min_{U}\ \left\langle UU^\top, \sum_{\ell=1}^T w_\ell^k L^{(\ell)} \right\rangle + \lambda\|UU^\top\|_1, \ \st, \ U\in\M, 
\ee
which is in the same form of \eqref{SSC-manopt-rewrite} with $L$ in \eqref{SSC-manopt-rewrite} replaced by $\bar{L}:=\sum_{\ell=1}^T w_\ell^k L^{(\ell)}$. Therefore, \eqref{prob:mkssc-U} can also be solved by ManPL. Here we adopt one step of ManPL, i.e., \eqref{ManPL-alg} to obtain $U^{k+1}$. More specifically, 
$U^{k+1}$ is computed by the following two steps:
\begin{subequations}\label{AManPL-alg-U}
\begin{align}
V^{k} & := \argmin_V \ \langle \nabla_U f(U^k,w^k), V \rangle + h(c(U^k) + J(U^k) V) + \frac{1}{2t}\|V\|_F^2, \ \st, \ V\in\T_{U^k}\M, \label{AManPL-alg-U-1}\\
U^{k+1} & := \Retr_{U^k}(\alpha_kV^{k}),\label{AManPL-alg-U-2}
\end{align}
\end{subequations}
where $f(U,w) := \left\langle UU^\top, \sum_{\ell=1}^T w_\ell L^{(\ell)} \right\rangle$, $h(\cdot):=\lambda\|\cdot\|_1$, and $c(U) = UU^\top$. Note that \eqref{AManPL-alg-U-1} can be solved by the same PPA+SSN algorithm discussed in Section \ref{sec:ssn}.
We then fix $U$ in \eqref{prob:mkssc} as $U^{k+1}$, and then \eqref{prob:mkssc} reduces to
\be\label{prob:mkssc-w}
\min_{w} \ c^\top w + \rho\sum_{\ell=1}^T w_\ell\log(w_\ell), \ \st, \ \sum_{\ell=1}^T w_\ell =1, w_\ell \geq 0, \ell=1,\ldots,T, 
\ee
where $c_\ell = \langle U^{k+1}{U^{k+1}}^\top, L^{(\ell)} \rangle$, $\ell=1,\ldots,T$. We then obtain $w^{k+1}$ by solving \eqref{prob:mkssc-w}, which admits a closed-form solution given by \eqref{prob:mkssc-w-sol}. 
The AManPL is described in Algorithm \ref{alg:MKSSC}.
\begin{algorithm}[ht]
\begin{algorithmic}
\STATE{Input: parameter $\gamma\in(0,1)$, initial point $U^0 \in \mathcal{M}$, let $w^0$ be the optimal solution to \eqref{prob:mkssc-w} for $c_\ell = \langle U^{0}{U^{0}}^\top, L^{(\ell)} \rangle$ }
\FOR{$k = 0, 1, \ldots$}
    \STATE{Calculate $V^{k}$ by solving \eqref{AManPL-alg-U-1}}
    \STATE{Let $j_k$ be the smallest nonnegative integer such that}
        \be\label{eq:line-search-mkssc}\bar{F}(\Retr_{U^k}(\beta^{j_k}V^{k}),w^k) \leq \bar{F}(U^k,w^k) - \frac{\gamma^{j_k}}{2t}\|V^{k}\|_F^2\ee
    \STATE{Let $\alpha_k = \gamma^{j_k}$ and compute $U^{k+1}$ by \eqref{AManPL-alg-U-2}}
    \STATE{Update $w_\ell^{k+1}$ by \eqref{prob:mkssc-w-sol} with $c_\ell = \langle U^{k+1}{U^{k+1}}^\top, L^{(\ell)} \rangle$, $\ell=1,\ldots,T$}
\ENDFOR 
\end{algorithmic}
\caption{The AManPL Method for Solving MKSSC \eqref{prob:mkssc}}\label{alg:MKSSC}
\end{algorithm}

We have the following convergence and iteration complexity result for AManPL for solving MKSSC \eqref{prob:mkssc}. Its proof is given in the appendix. 
\begin{theorem}\label{thm:amanpl-convergence}
	Assume $\bar F(U,w)$ in \eqref{prob:mkssc} is lower bounded by $\bar F^*$. 
	The limit point of the sequence $\{U^{k},w^k\}$ generated by AManPL (Algorithm \ref{alg:MKSSC}) is a stationary point of problem \eqref{prob:mkssc}. Moreover, to obtain an $\epsilon$-stationary point of problem \eqref{prob:mkssc}, the number of iterations needed by AManPL is $O(\epsilon^{-2})$.  
\end{theorem}

\section{Numerical Experiments}\label{simulation}

In this section, we compare our proposed methods ManPL and AManPL with some existing methods for solving SSC and MKSSC. In particular, for SSC \eqref{SSC}, we compare ManPL (Algorithm \ref{alg:manpl-ssc}) with convex ADMM \citep{lu2016convex} (denoted by CADMM \footnote{cdoes downloaded from \url{https://github.com/canyilu/LibADMM/blob/master/algorithms/sparsesc.m}}) for solving \eqref{SSC-convex-relax-Lu} and nonconvex ADMM \citep{lu2018nonconvex} (denoted by NADMM) for solving \eqref{SSC-nonconvex-Lu}. We also include the spectral clustering (denoted by SC) in the comparison. For MKSSC \eqref{prob:mkssc}, we compare AManPL (Algorithm \ref{alg:MKSSC}) with MPSSC (i.e., AMA+CADMM \footnote{codes downloaded from \url{https://github.com/ishspsy/project/tree/master/MPSSC}}) \citep{park2018spectral} and AMA+NADMM \eqref{prob:mkssc-smooth-AMA+NADMM}. 
All the algorithms were terminated when the absolute change of the objective value is smaller than $10^{-5}$, which indicates that the algorithms were not making much progress. 

\begin{table}[ht]
\begin{center}
\begin{tabular}{|c|c|c|}\hline
Problem & Algorithm & Parameters \\\hline\hline
Convex SSC \eqref{SSC-convex-relax-Lu} & CADMM \citep{lu2016convex} & $\lambda = 10^{-4}$ \\\hline
Smoothed SSC \eqref{SSC-nonconvex-Lu} & NADMM \eqref{eq:non-convex:admm} \citep{lu2018nonconvex} & $\lambda=10^{-4}, \sigma =10^{-2}$ \\\hline
Original SSC \eqref{SSC} & ManPL (Algorithm \ref{alg:manpl-ssc}) & $\lambda = 10^{-3}$ \\\hline\hline
MKSSC \eqref{prob:mkssc-park} & AMA+CADMM \eqref{AMA-park-zhao} \citep{park2018spectral} & $\lambda=10^{-4}$, $\rho=0.2$ \\\hline
Smoothed MKSSC \eqref{prob:mkssc-smooth} & AMA+NADMM \eqref{prob:mkssc-smooth-AMA+NADMM} & $\lambda = 10^{-4}, \rho =0.2 , \sigma =10^{-2}$ \\\hline
Original MKSSC \eqref{prob:mkssc} & AManPL (Algorithm \ref{alg:MKSSC}) & $\lambda=5\times 10^{-3}$, $\rho=1$\\\hline
\end{tabular}
\caption{Algorithms and their parameters}\label{tab:param}
\end{center}
\end{table}


\subsection{UCI Datasets}
We first compare the clustering performance of different methods on three benchmark datasets in UCI machine learning repository \citep{Dua:2019}. We list the parameters used in the algorithms in Table \ref{tab:param}. {These parameters are either suggested in the corresponding papers, or are tuned from a set of values and the ones that yielded the best normalized mutual information (NMI) score (see Table \ref{tab:uci}) were chosen. For instance, $\lambda$ in CADMM was chosen from the set $\{10^{-3},10^{-4},10^{-5}\}$ and this was suggested in \citep{lu2016convex}. The values of $\lambda$ and $\rho$ in AMA+CADMM were suggested in \citep{park2018spectral}.} We follow \citep{park2018spectral} to construct the similarity matrices and  record the NMI scores to measure the performance of the clustering. Note that higher NMI scores indicate better clustering performance. The NMI scores are reported in Table \ref{tab:uci}. 
From Table \ref{tab:uci} we see that the three algorithms for SSC usually outperform SC, and among the three algorithms for SSC, ManPL usually outperforms the other two methods. We also see that the MKSSC model usually generates higher NMI scores than SSC. Moreover, among all three algorithms for SSC and three algorithms for MKSSC, the AManPL for MKSSC always has the largest NMI score. This indicates the great potential of our AManPL method for solving MKSSC. Morevoer, we show the heatmap of $|UU^\top|$ for the Wine data set in Figure \ref{fig:UCI-UUT}. From this figure we see that for SSC, the NADMM and ManPL generate much better results than CADMM in terms of recovering the block matrices, and for MKSSC, the AMA+NADMM and AManPL generate much better results than AMA+CADMM. The heatmaps for the other two data sets are similar, so we omit them for brevity.  


\begin{table}[ht]
\resizebox{\columnwidth}{!}{
\begin{tabular}{|c||c||ccc|ccc||c|}
\hline
  &  & \multicolumn{3}{c|}{SSC} & \multicolumn{3}{c|}{MKSSC} & \\\hline
Datasets      & SC     & CADMM   & NADMM  & ManPL  & AMA+CADMM & AMA+NADMM & AManPL & C \\
\hline
Wine  & 0.8650 & 0.8650 & 0.8782 & 0.8782 & 0.8854 & 0.8854 &\textbf{0.8926} & 3\\
Iris  & 0.7496 & 0.7582 & 0.7582 & 0.7582 & 0.7665 & 0.7601 &\textbf{0.7705} & 3\\
Glass & 0.3165 & 0.3418 & 0.2047 & 0.3471 & 0.2656 & 0.3315& \textbf{0.3644} & 6\\
\hline
\end{tabular}
}
\caption{Comparison of the NMI scores on the UCI data sets.}\label{tab:uci}
\end{table}

\begin{figure}[ht]
    \centering
    \subfloat[]{{\includegraphics[width=0.2\linewidth]{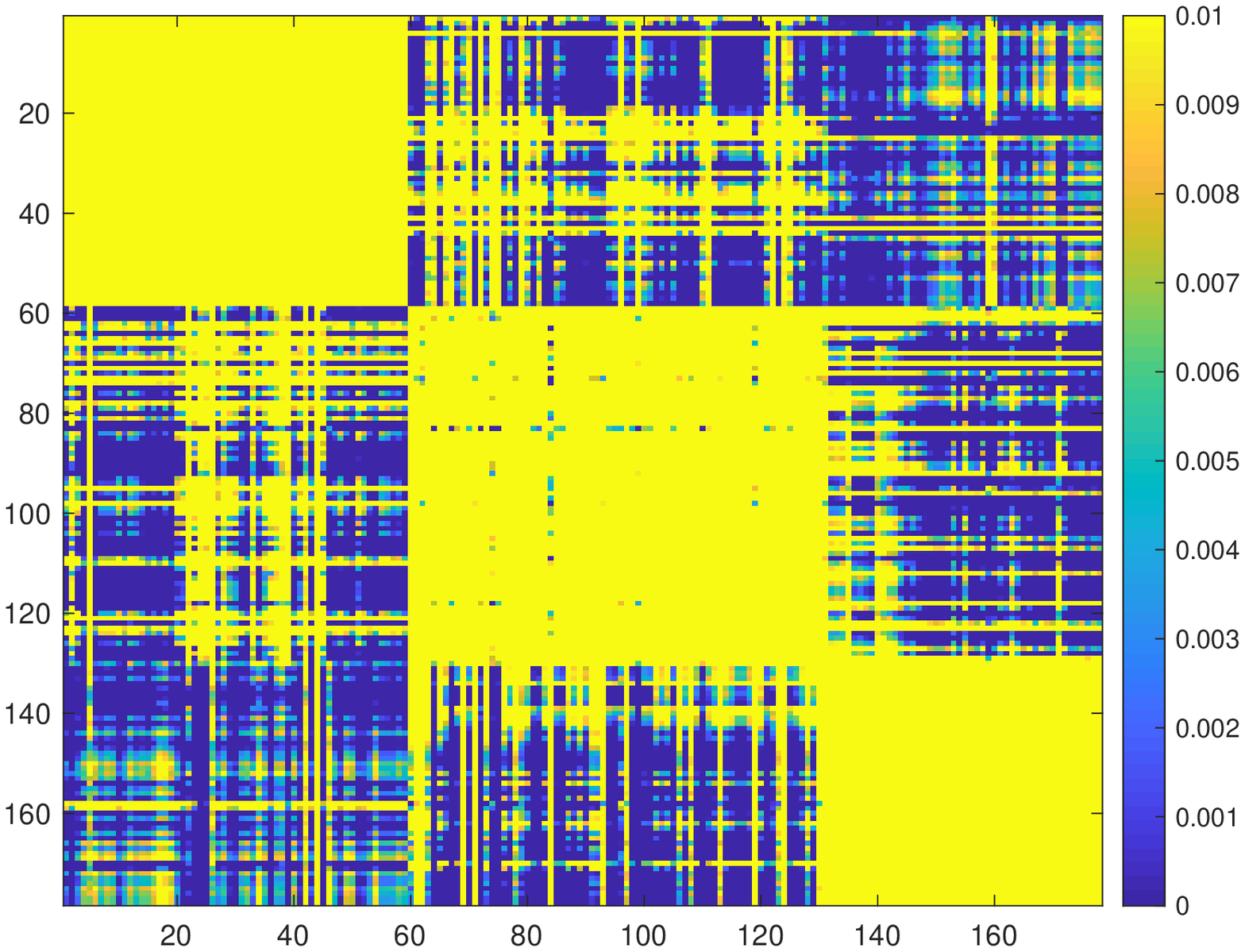} }}%
    \quad
    \subfloat[]{{\includegraphics[width=0.2\linewidth]{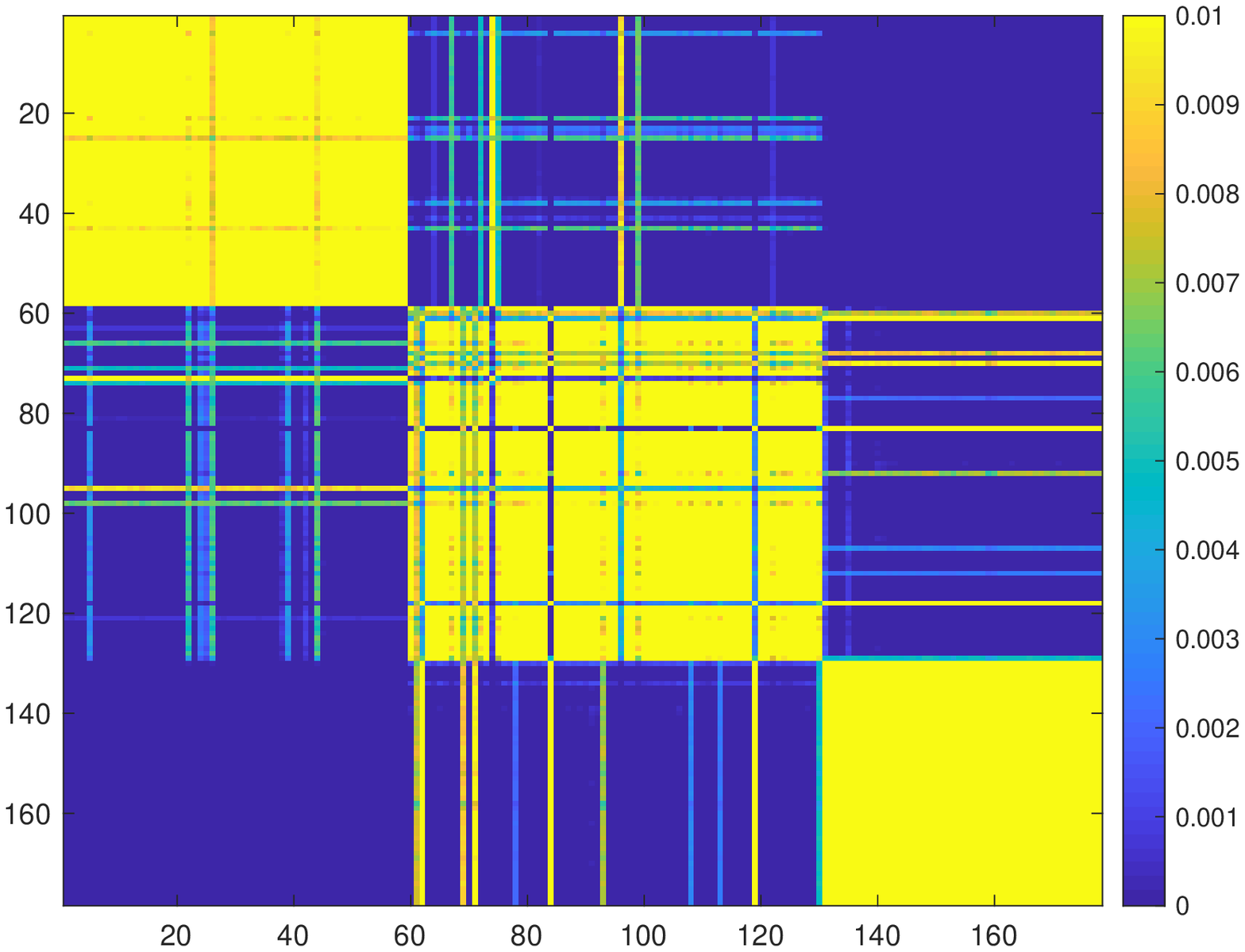} }}%
    \quad
    \subfloat[]{{\includegraphics[width=0.2\linewidth]{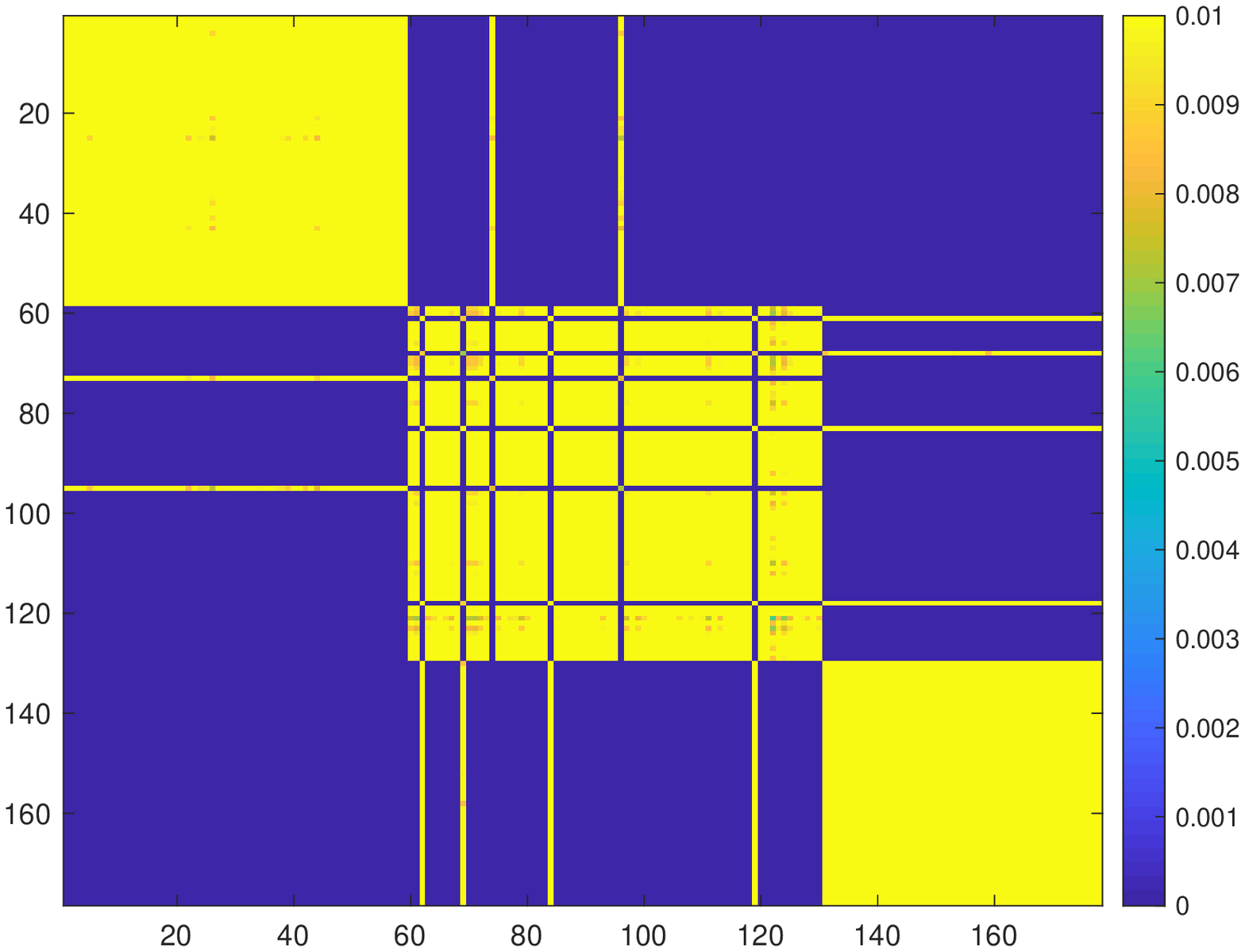} }}%
    \quad
    \subfloat[]{{\includegraphics[width=0.2\linewidth]{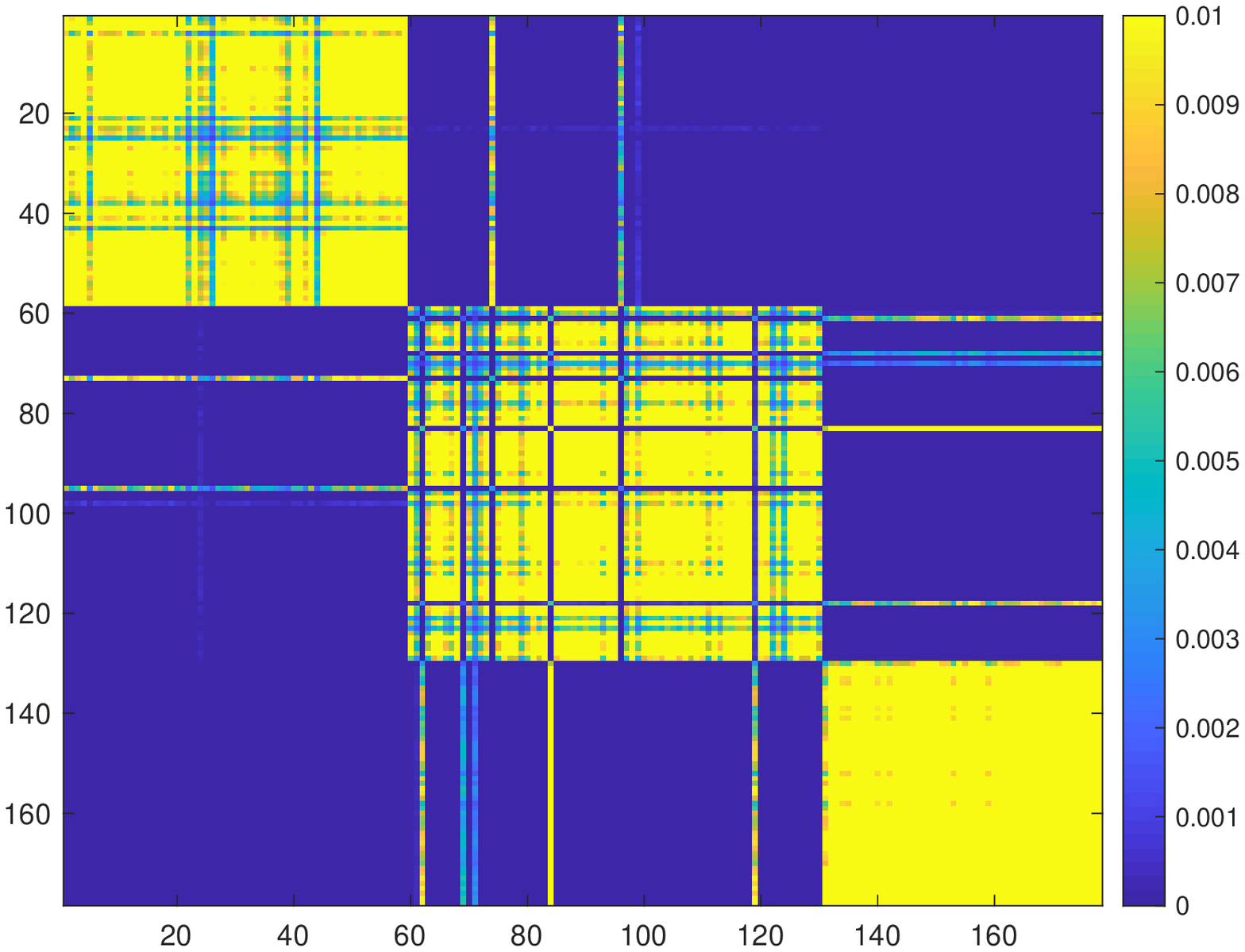} }}%
    \quad
    \subfloat[]{{\includegraphics[width=0.2\linewidth]{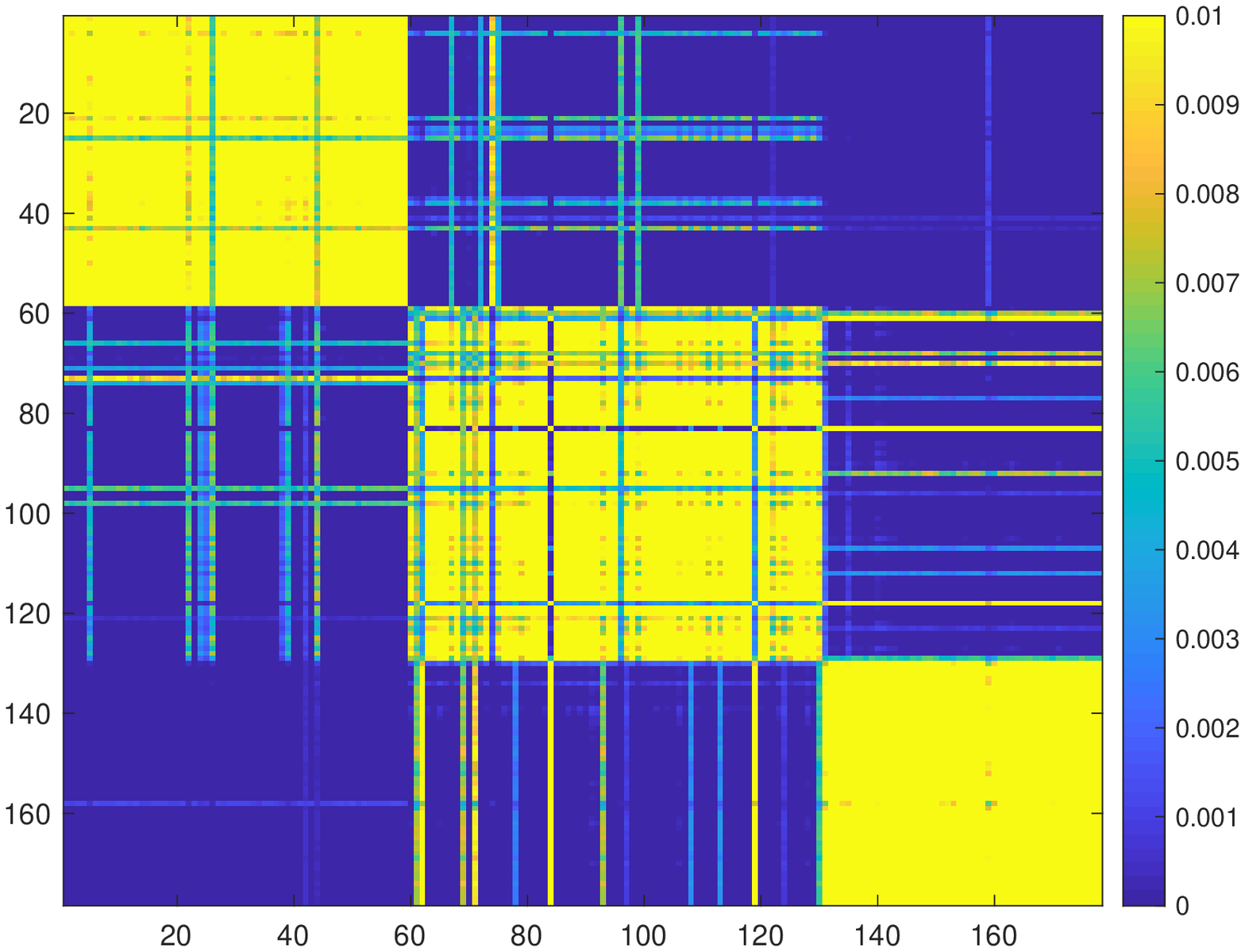} }}%
    \quad
    \subfloat[]{{\includegraphics[width=0.2\linewidth]{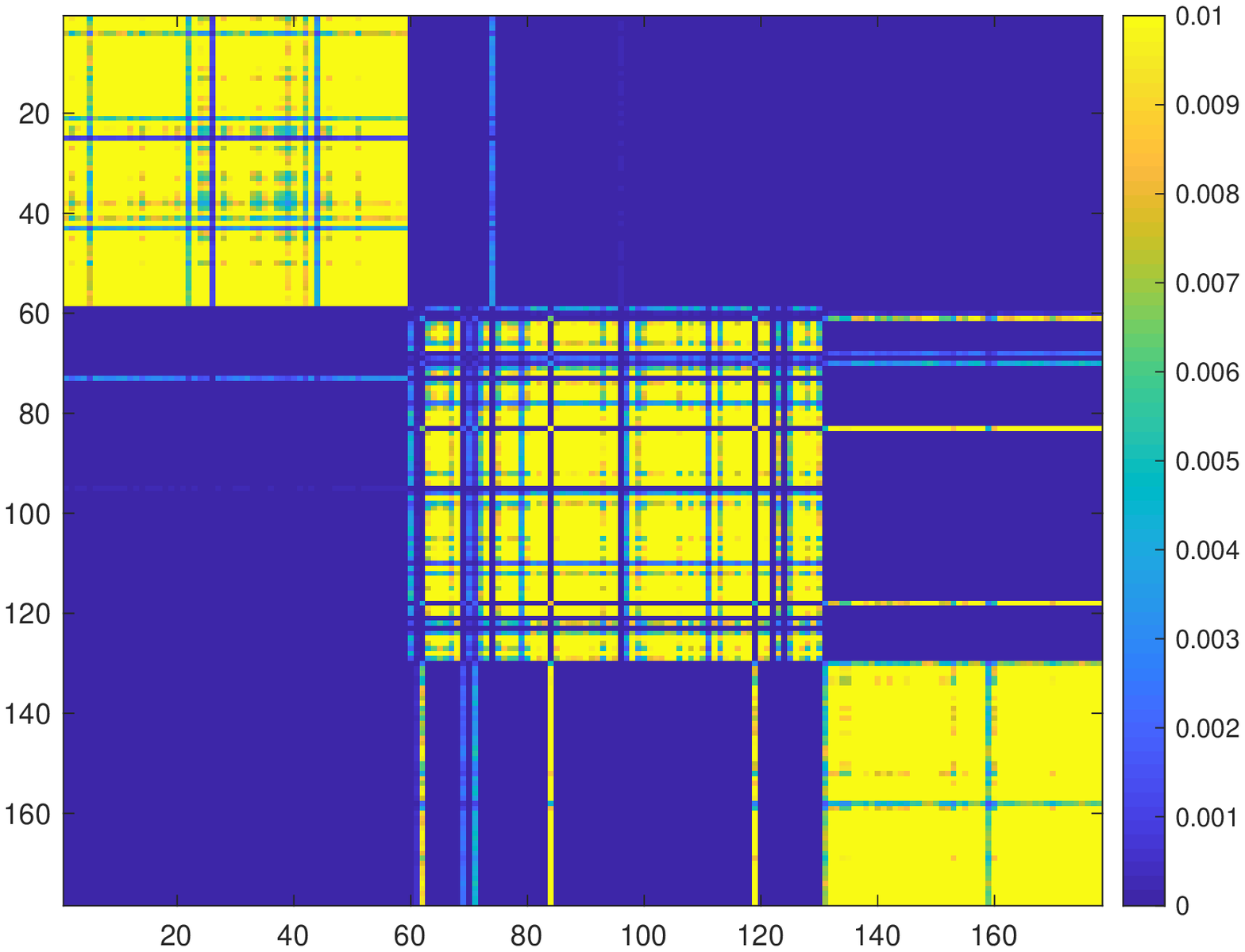} }}%
    \quad
    \subfloat[]{{\includegraphics[width=0.2\linewidth]{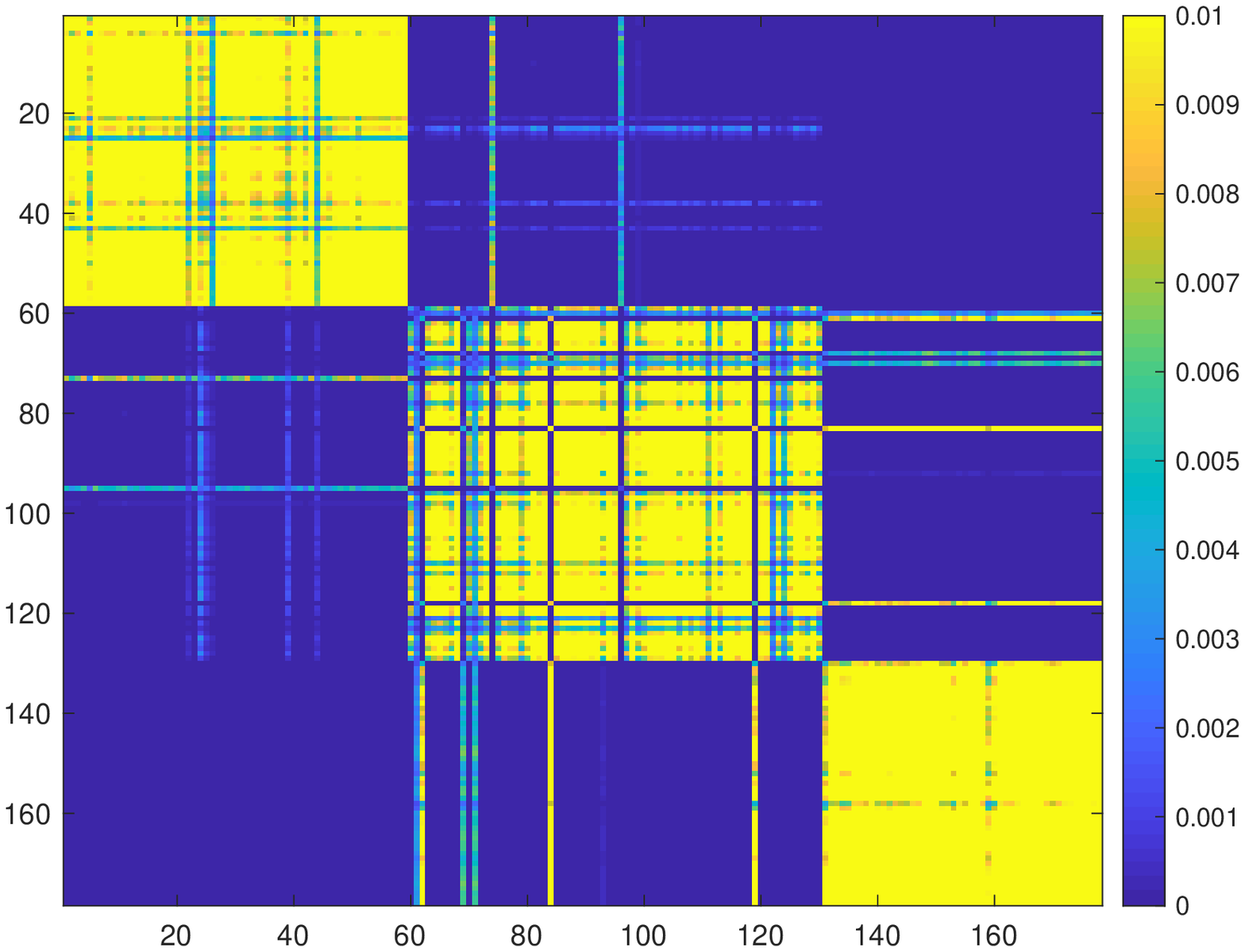} }}%
    \caption{The  heatmaps of $|\bU\bU^{\top}|$ on the Wine dataset estimated by (a) SC, (b) CADMM for SSC, (c) NADMM for SSC, (d) ManPL for SSC, (e) AMA+CADMM for MKSSC, (f) AMA+NADMM for MKSSC, and (g) AManPL for MKSSC. 
    }%
    \label{fig:UCI-UUT}%
\end{figure}

\subsection{Synthetic Data}
In this subsection, we follow \citep{park2018spectral} to evaluate the clustering performance of different methods on two synthetic datasets with $C=5$ clusters. {To compare ManPL, AManPL and existing methods, we select the regularization parameter $\lambda$ using an independent tuning dataset and set other parameters as in Table \ref{tab:param}. Specifically, we generate the tuning datasets using the same data generating process, and we select the parameter $\lambda$ that maximizes the average NMI score over 50 independent repetitions. The two synthetic datasets are generated as follows.}

\begin{itemize}
\item \textbf{Synthetic data 1.}  We randomly generate \(C\) points in the $2$-dimensional latent space spanning a  circle as the centers of \(C\) clusters. For each cluster, we randomly generate the points by adding  an independent noise to its center. We project these \(2\)-dimensional data to a $p$-dimensional space using a  linear projection matrix 
and then add the heterogeneous noise 
to obtain the data matrix \(X\). The noise level is $30\%$ of the radius of the circle in the embedded space.
\item \textbf{Synthetic data 2.}  We randomly generate a  matrix \(B^{\prime} \in \mathbb{R}^{C\times d}\) by drawing its entries independently from  Gaussian distributions, where \(d<p\) and different rows of $B^{\prime}$ specify heterogeneous variances. We randomly assign the cluster labels \(z_{1},\ldots, z_n \in[C]\). Let \(B=\left[B^{\prime}, 0_{C \times(p-d)}\right]\) and \(Z=(Z_{i j})_{n \times C}=(1_{\left\{z_{i}=j\right\}})_{n \times C}\).
We generate \(X=Z B+W,\) where \(W\) is a  noise matrix with independent standard normally distributed entries. The noise level is $20\%$ of the radius of the circle in the embedded space.
\end{itemize}

Figure \ref{fig:artificial example} visualizes one realization of the simulated data for these two settings. From Figure \ref{fig:artificial example} we see that different clusters mix together and the variability between clusters varies. Since we found that MKSSC is better to handle the heterogeneous noise than SSC in both settings, we only focus on the comparison of AMA+CADMM, AMA+NADMM and AManPL for MKSSC. In Table \ref{table_nmi}, we report the NMI scores of the three algorithms for solving MKSSC. The NMI scores are averaged over 50 independent runs and the numbers in the parentheses are the standard deviation of the NMI scores. From Table \ref{table_nmi} we see that AManPL consistently outperforms the other two methods in all tested instances. This is not surprising because AManPL solves the original MKSSC, and the other two methods only solve its approximations. Moreover, we show the heatmaps of $|UU^{\top}|$ for synthetic data 1 in Figure \ref{figure_s1} (synthetic data 2 is simliar). We see that the block diagonal structure is well recovered by AManPL, which shows much better performance than the other two methods. 

\begin{figure}[ht]
    \centering
    \subfloat[Synthetic data 1]{{\includegraphics[width=0.4\linewidth]{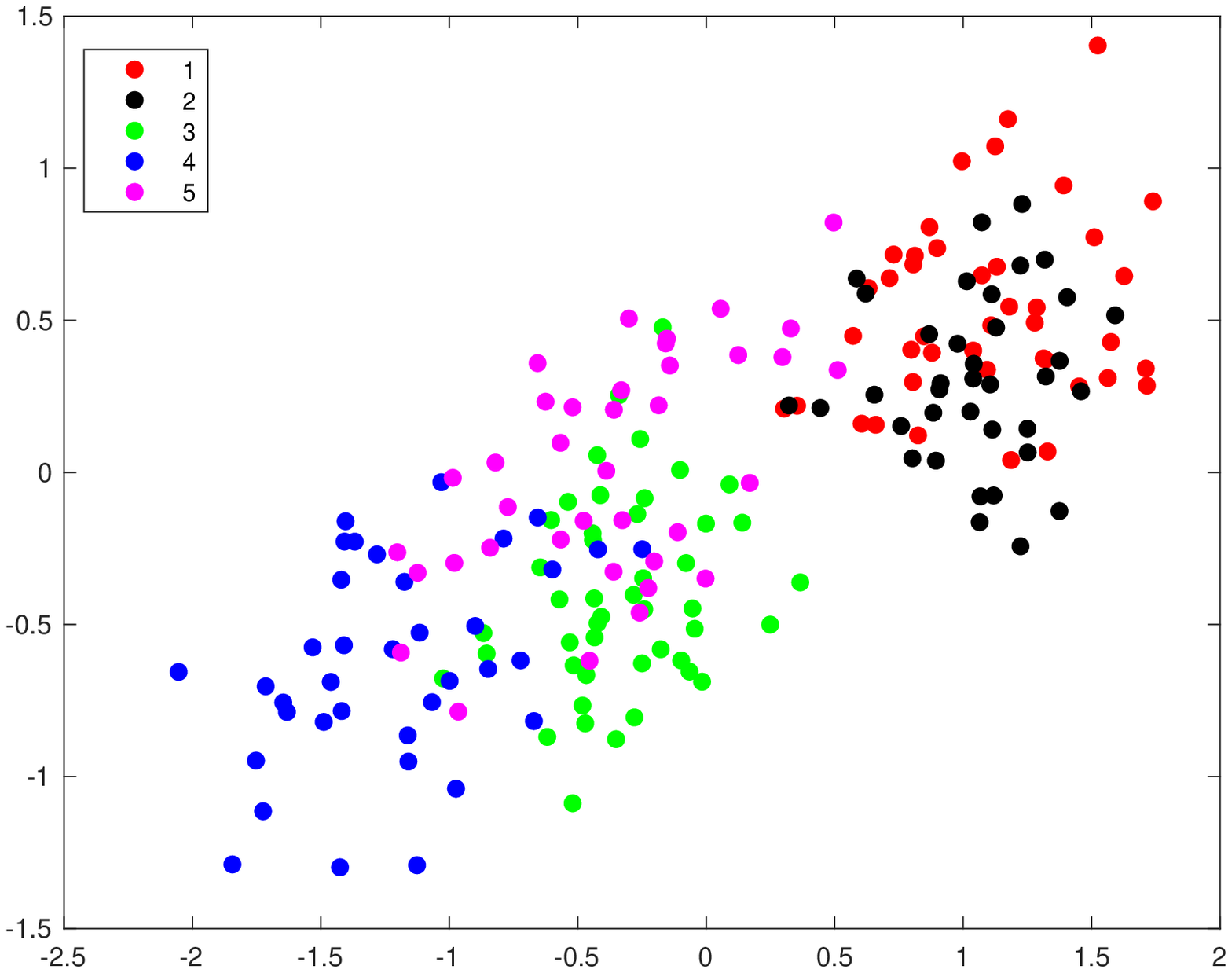} }}%
    \qquad
    \subfloat[Synthetic data 2]{{\includegraphics[width=0.4\linewidth]{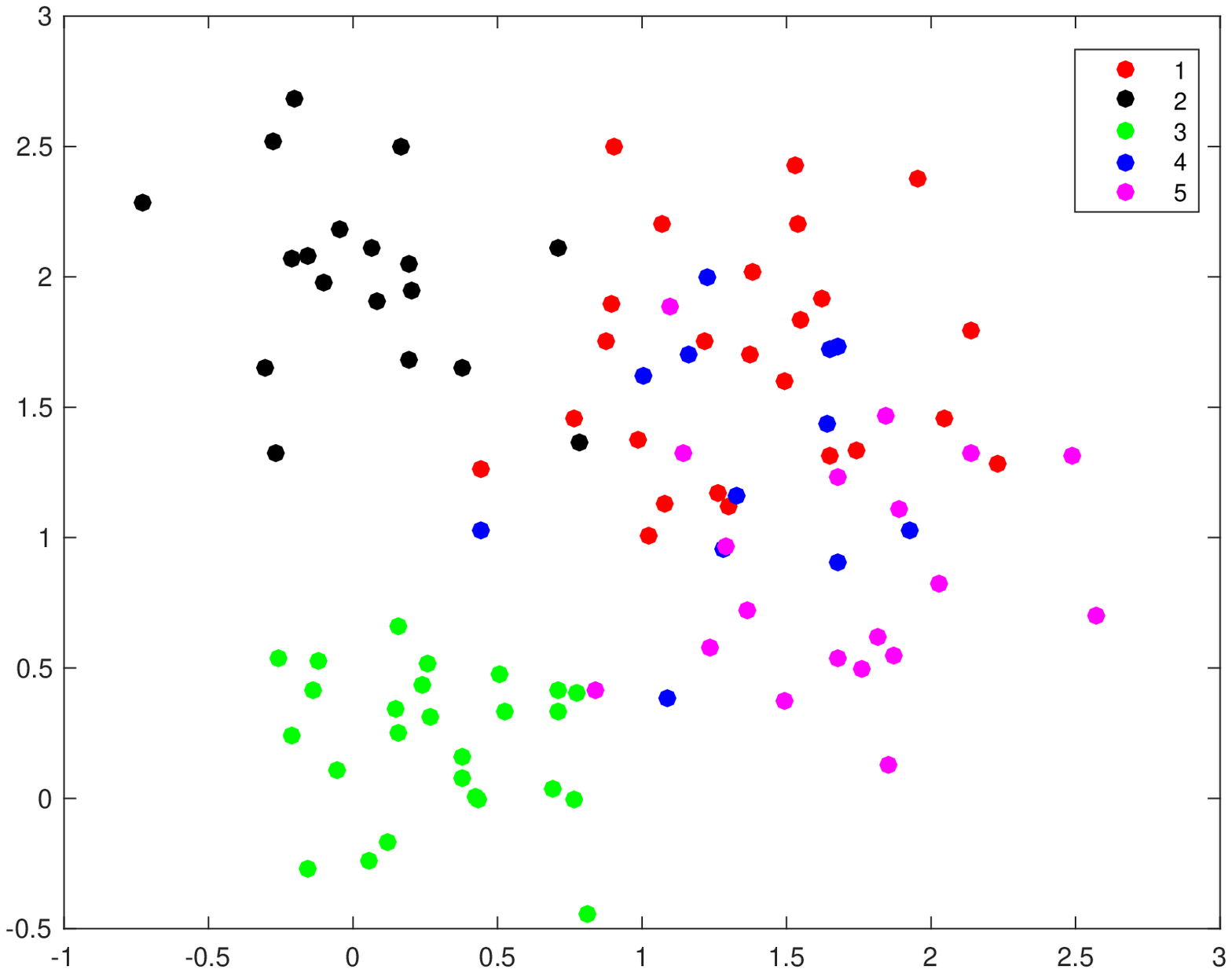} }}%
    \caption{Illustration of one realization of the synthetic data.}%
    \label{fig:artificial example}%
\end{figure}


\begin{table}[ht]
\centering
\begin{tabular}{|cc|l|l|l|}
\hline
 \multicolumn{2}{|c|}{Method}    & AMA+CADMM  & AMA+NADMM & AManPL \\
\hline
\hline
$n$   & $p$ & \multicolumn{3}{|c|}{Synthetic data 1}\\
\hline
100 & 250 & 0.9852 (1.2e-3) & 0.9852 (2.4e-3) & 0.9933 (1.7e-4) \\
100 & 300 & 0.9789 (2.1e-2) & 0.9844 (1.4e-2) & 0.9917 (7.9e-4)  \\
100 & 500 & 0.9787 (2.0e-3) & 0.9834 (1.5e-3) & 0.9956 (1.2e-4)  \\
200 & 250 & 0.9821 (1.8e-3) & 0.9803 (1.9e-3) & 0.9955 (6.6e-5)  \\
200 & 300 & 0.9830 (1.3e-3) & 0.9833 (1.3e-3) & 0.9844 (1.6e-4) \\
200 & 500 & 0.9607 (2.8e-3) & 0.9606 (2.8e-3) & 0.9867 (3.6e-4) \\
\hline
\hline
$n$   & $p$ & \multicolumn{3}{|c|}{Synthetic data 2}\\
\hline
100 & 250 & 0.6491 (5.8e-3) & 0.7163 (3.9e-3) & 0.7334 (7.6e-3)  \\
100 & 300 & 0.6304 (1.3e-3) & 0.7466 (2.4e-3) & 0.7881 (1.5e-3)  \\
100 & 500 & 0.6253 (2.3e-3) & 0.7289 (1.3e-3) & 0.7308 (1.1e-3)  \\
200 & 250 & 0.7977 (2.1e-3) & 0.1371 (2.2e-3) & 0.9182 (3.2e-4) \\
200 & 300 & 0.7380 (1.1e-3) & 0.1034 (1.2e-3) & 0.8773 (1.2e-3)  \\
200 & 500 & 0.7130 (9.6e-3) & 0.1220 (2.1e-3) & 0.8199 (4.5e-3) \\
\hline
\end{tabular}
\caption{NMI of three algorithms for solving MKSSC for synthetic data.}\label{table_nmi}
\end{table}


\begin{figure*}[ht]
\centering
    \subfloat[]{{\includegraphics[width=0.3\linewidth]{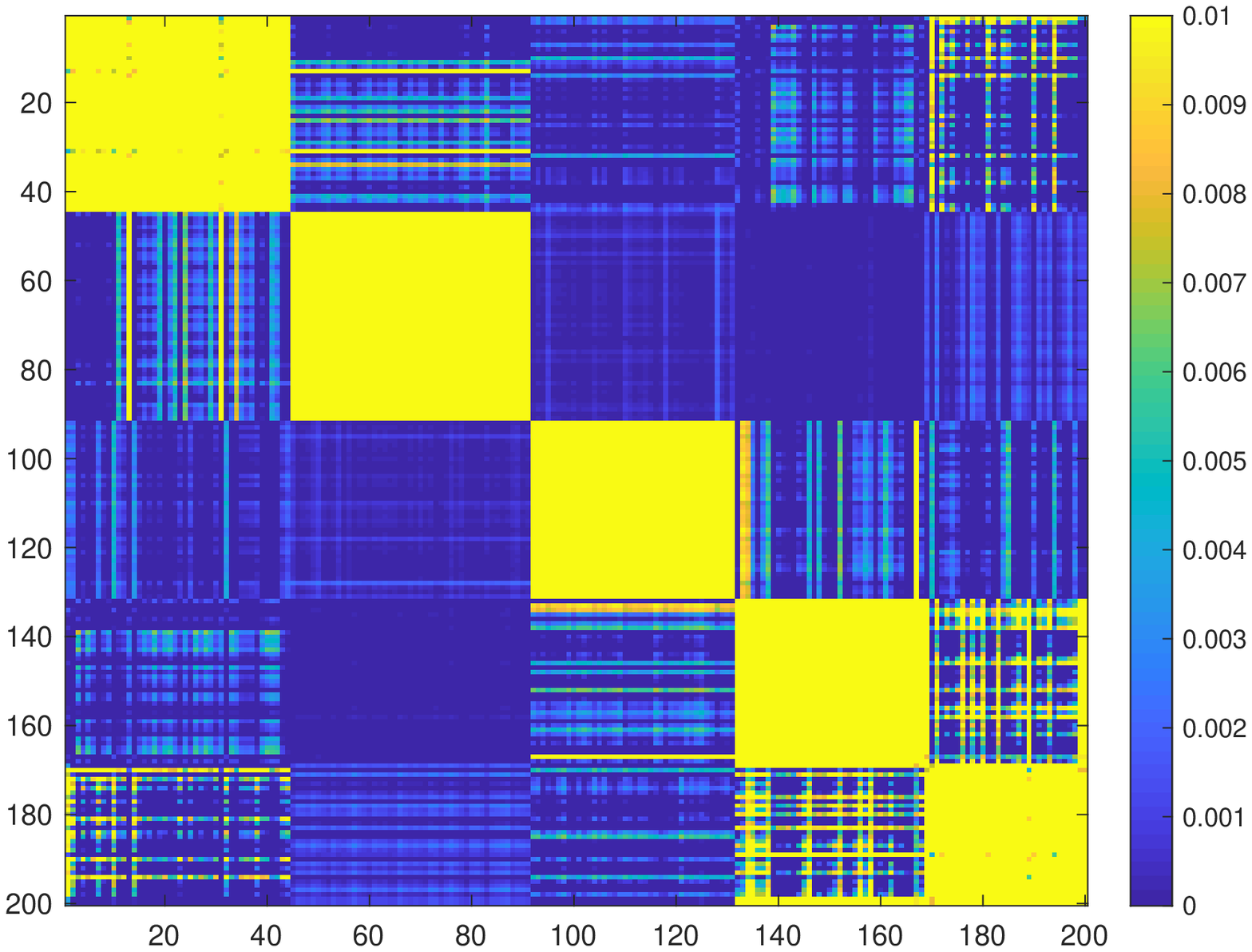} }}%
    \quad
    \subfloat[]{{\includegraphics[width=0.3\linewidth]{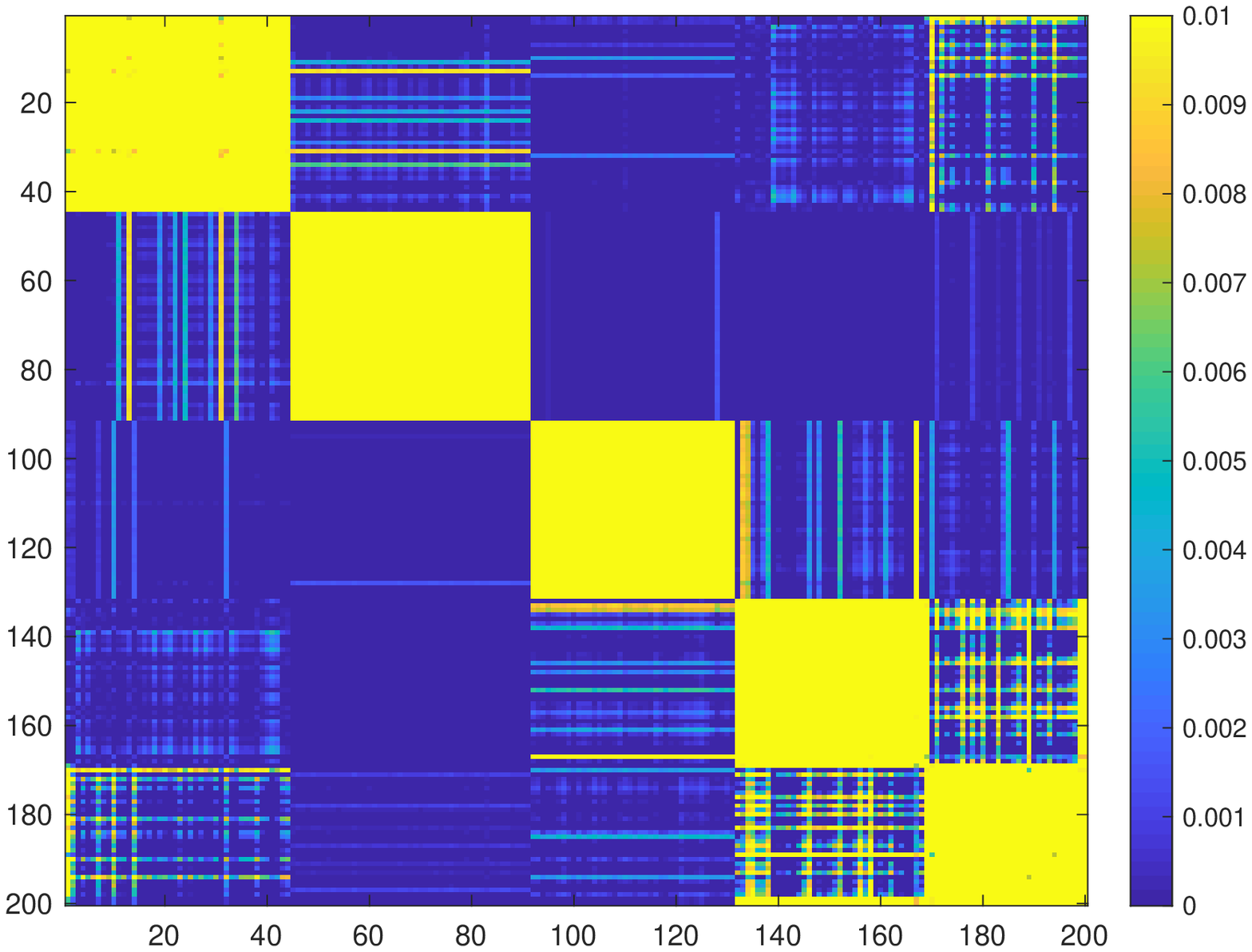} }}%
    \quad
    \subfloat[]{{\includegraphics[width=0.3\linewidth]{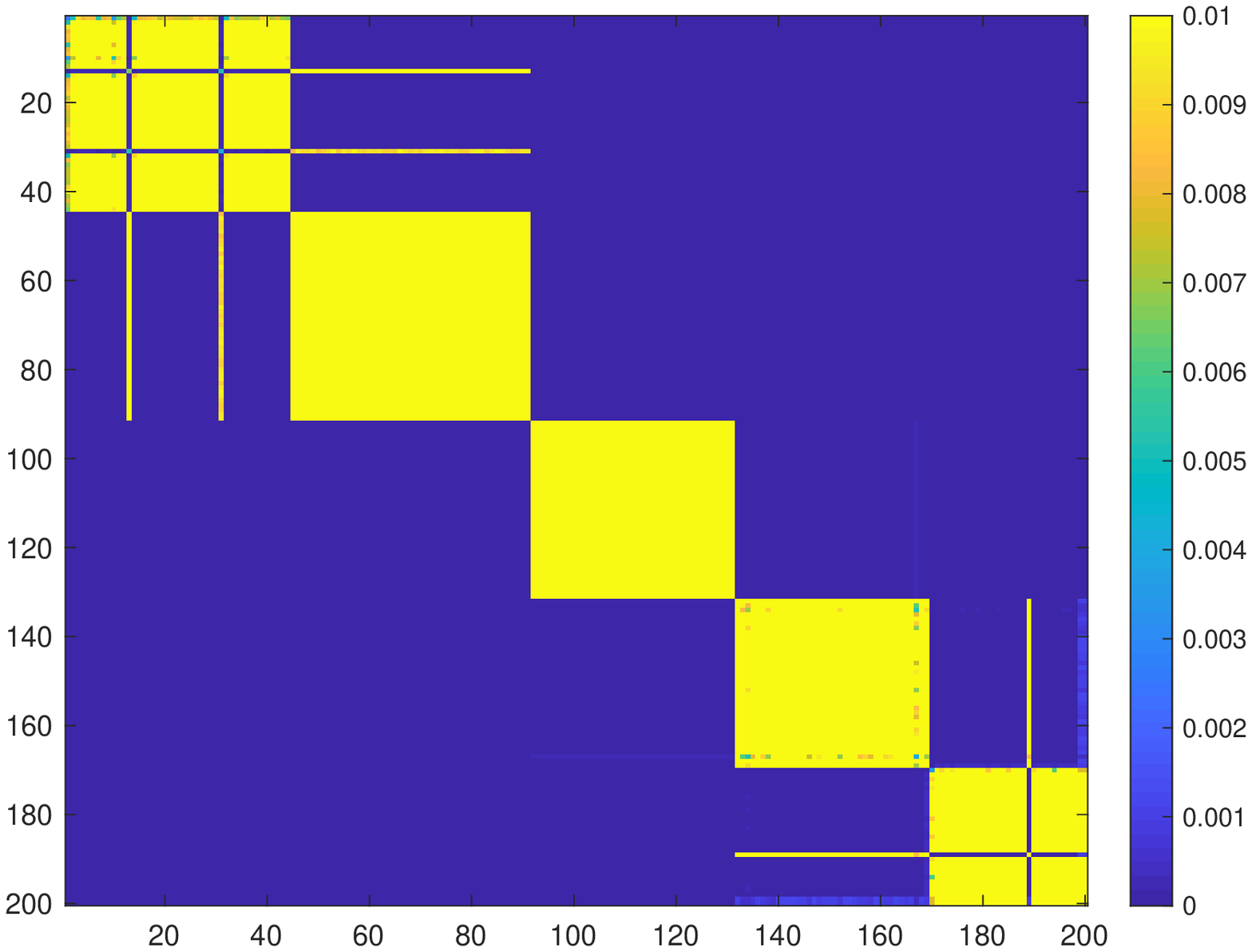} }}%
\caption{The heatmaps of $|\bU\bU^{\top}|$ on synthetic data 1 estimated by (a)  AMA+CADMM, (b) AMA+NADMM, and (c) AManPL for MKSSC.}
\label{figure_s1}
\end{figure*}

\subsection{Single-Cell RNA Sequencing Data Analysis}\label{real-application}

Clustering cells and identifying subgroups are important topics in high-dimensional scRNA-seq data analysis. The multiple kernel learning approach is vital as clustering scRNA-seq data is usually sensitive to the choice of the number of neighbors and scaling parameter. Recently, \citep{park2018spectral} showed  that AMA+CADMM for MKSSC provides a promising clustering result and outperforms several state-of-art methods such as SC, SSC, t-SNE \citep{maaten2008visualizing}, and SIMLR \citep{wang2017simlr}. In what follows, we focus on the numerical comparison of AMA+CADMM, AMA+NADMM and AManPL to cluster high-dimensional scRNA-seq data on six real datasets. These six real datasets represent several types of important dynamic processes such as cell differentiation, and they include the information about single cell types. We follow the procedure of \citep{park2018spectral} to specify multiple kernels for clustering scRNA-seq data and choose the proper tuning parameters $\lambda$ and $\rho$ from a set of given values so that the best NMI scores were obtained in each case. The six datasets and the NMI scores of the three algorithms: AMA+CADMM, AMA+NADMM and AManPL for solving MKSSC, are summarized in Table \ref{tab: real date results}. From Table \ref{tab: real date results} we observe that AManPL always achieves the highest NMI scores and consistently outperforms the other two methods on all six real datasets. We also demonstrate the performance of AManPL by showing the two-dimensional embedding estimated by AManPL in Figure \ref{fig:single cell clustering}. From this figure, we see that AManPL yielded clear and meaningful clusters, even for the Ginhous dataset \citep{schlitzer2015identification}, which was known to be very challenging. This again demonstrates the great practical potential of our AManPL method for analyzing the scRNA-seq data. 

\begin{table}[ht]
\centering
\begin{tabular}{|l||ccc||c|}
\hline
 Datasets   & AMA+CADMM  & AMA+NADMM & AManPL & $C$ \\
 \hline
Deng \citep{deng2014single}  & 0.7319 & 0.7389 & \textbf{0.7464}  & 7        \\
Ting \citep{ting2014single}   & 0.9283 & 0.9524 & \textbf{0.9755}  & 5        \\
Treutlein \citep{treutlein2014reconstructing} & 0.7674 & 0.7229 & \textbf{0.8817}  & 5        \\
Buettner \citep{buettner2015computational} & 0.7929 & 0.8744 & \textbf{0.8997}  & 3        \\
Ginhoux \citep{schlitzer2015identification}  & 0.6206 & 0.6398 & \textbf{0.6560}  & 3       \\
Pollen \citep{pollen2014low} & 0.9439 & 0.9372 & \textbf{0.9631} & 11       \\
\hline
\end{tabular}
\caption{NMI scores of three algorithms for MKSSC on six real scRNA-seq datasets.}\label{tab: real date results}
\end{table}

\begin{figure}[ht]
    \centering
    \subfloat[Ting \citep{ting2014single}]{{\includegraphics[width=0.4\linewidth]{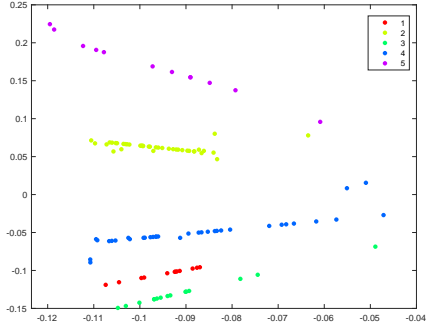} }}%
    \qquad
    \subfloat[Ginhoux \citep{schlitzer2015identification}]{{\includegraphics[width=0.4\linewidth]{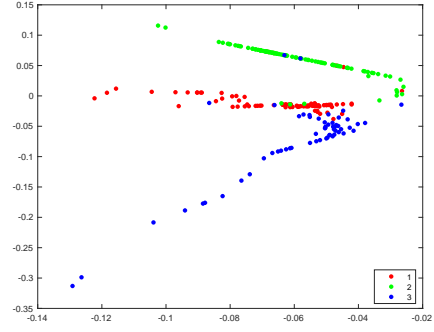} }}%
	\caption{Visualization of the cells in 2-D embedded space using AManPL for MKSSC on two real datasets in high-dimensional scRNA-seq analysis.}\label{fig:single cell clustering}%
\end{figure}

\section{Conclusion}\label{conclusion}

   Motivated by the recent need on analyzing single cell RNA sequencing data, we considered the sparse spectral clustering and multiple-kernel sparse spectral clustering in this paper. We proposed a manifold proximal linear method for solving SSC, and the alternating manifold proximal linear method for solving MKSSC. Convergence and iteration complexity of the proposed methods are analyzed. Numerical results on synthetic data and the single cell RNA sequencing data demonstrated the great potential of our proposed methods. 
   


\appendix
\renewcommand{\thesection}{\Alph{section}}
\setcounter{section}{0}

\section{Preliminaries}\label{sec:appen-pre}
 \subsection{Optimality Conditions of Manifold Optimization}
\begin{definition}(Generalized Clarke subdifferential \citep{hosseini2011generalized})
	For a locally Lipschitz function $F$ on $\M$, the Riemannian generalized directional derivative of $F$ at $X\in\M$ in direction $V$ is defined by
	\be\label{r_dir_derivative}
	F^{\circ}(X,V) =\limsup\limits_{Y\rightarrow X,t\downarrow 0}\frac{F\circ \phi^{-1}(\phi(Y)+tD\phi(X )[V])- {F}\circ \phi^{-1}(\phi(Y))}{t},
	\ee
	where $(\phi,U)$ is a coordinate chart at $X$ and $D\phi(X)$ denotes the Jacobian  of $\phi(X)$.
	The generalized gradient or the Clarke subdifferential of $F$ at $X\in\M$, denoted by ${\partial}_R F(X)$, is given by
	\be\label{r_clarke_sub}
	 {\partial}_R F(X)=\{\xi\in \T_X\M :\inp{\xi}{V}\leq F^{\circ}(X,V), \ \forall V\in \T_X\M \}.
	\ee
\end{definition}
\begin{definition}(\citep{yang2014optimality})
	A function $f$ is said to be regular at $X\in\M$ along $\T_X\M$ if
	\begin{itemize}
		\item for all $V\in \T_X\M$, $f'(X;V)=\lim_{t\downarrow 0} \frac{f(X+tV)-f(X)}{t}$ exists, and
		\item for all $V\in \T_X\M$, $f'(X;V) = f^\circ (X;V)$.
	\end{itemize}
\end{definition}

For a smooth function $f$ over Riemannian submanifold, if the metric on the manifold is induced by the Euclidean inner product in the ambient space, then we know that $\grad f(X)= \Proj_{\T_X\M} \nabla f(X)$. Here $\grad f$ denotes the Riemannian gradient of $f$, and $\Proj_{\T_X\M}$ denotes the projection onto ${\T_X\M}$. According to Lemma 5.1 in \citep{yang2014optimality}, for a regular function 
$F$, we have ${\partial}_R F(X)=\Proj_{\T_X\M}(\partial F(X)).$ Moreover, the function $h(c(U))$  in problem \eqref{SSC-manopt-rewrite} is weakly convex and thus is regular according to Lemma 5.1 in \citep{yang2014optimality}. 
 By Theorem 4.1 in \citep{yang2014optimality}, the first-order optimality condition of problem \eqref{SSC-manopt-rewrite} is given by
\be\label{opt-cond}0 \in \Proj_{\T_{U}\M}{\nabla c(U)}^\top\partial h(c(U)) + \grad f(U).\ee

\begin{definition}\label{stationary_point}
	A point $U\in\M$ is called a stationary point of problem \eqref{SSC-manopt-rewrite} if it satisfies the first-order optimality condition \eqref{opt-cond}. 
\end{definition}
\begin{definition}
    A retraction on a differentiable manifold $\mathcal{M}$ is a smooth mapping Retr from the tangent bundle $\T\mathcal{M}$ onto $\mathcal{M}$ satisfying the following two conditions (here ${\Retr}_{X}$ denotes the restriction of $ {\Retr}$ onto $\T_X\mathcal{M}$)
    \begin{itemize}
        \item $ {\Retr}_X(0) = X$, $\forall X \in \mathcal{M}$, where $0$ denotes the zero element of $\T_X \mathcal{M}$.
        \item For $X \in \mathcal{M}$, it holds that \[
        \lim_{\text{T}_X \mathcal{M}\ni\xi\to 0} \frac{\| {\Retr}_X(\xi) - (X + \xi)\|_F}{\|\xi\|_F} = 0.
        \]
    \end{itemize}
    The retraction onto the Euclidean space is simply the indentity mapping: ${\Retr}_X(\xi) = X + \xi$. Common retractions include the polar decomposition:
    \[
    \Retr_X^{\text{polar}}(\xi) = (X + \xi) (I_r + \xi^{\top} \xi)^{-1/2},
    \]
    the QR decomposition:
    \[
    \Retr_X^{\text{QR}} = \textbf{qf}(X + \xi),
    \]
    where $\textbf{qf}(A)$ is the $Q$ factor of the QR factorization of $A$, and the Cayley transformation:
    \[
    \text{Retr}_{X}^{\text{cayley}}(\xi) = \Bigl(I_n - \frac{1}{2}W(\xi)\Bigr)^{-1}\Bigl(I_n + \frac{1}{2} W(\xi)\Bigr)X,
    \]
    where $W(\xi) = (I_n - \frac{1}{2}XX^{\top})\xi X^{\top} - X\xi^{\top} (I_n - \frac{1}{2}XX^{\top})$
    \end{definition}
    
    The following lemma is useful in our analysis. 
\begin{lemma}\label{lem:two-inequalities}\citep{boumal2018global,chen2018proximal}
For all $X \in  { \mathcal { M } }$ and $\xi \in \T_X\cM$ there exist constants $M _ { 1 } > 0$ and $M _ { 2 } > 0$ such that the following two inequalities hold:
\begin{subequations}\label{apd:retrlemma}
\begin{align}\label{apd:Retrbound1}
\left\| \operatorname { Retr } _ { X } ( \xi ) - X \right\| _ { F } \leq M _ { 1 } \| \xi \| _ { F } , \forall X \in   { \mathcal { M } } , \xi\in \T_X\cM \\
\left\| \operatorname { Retr } _ { X } ( \xi ) - ( X + \xi ) \right\| _ { F } \leq M _ { 2 } \| \xi \| _ { F } ^ { 2 } , \forall X \in   { \mathcal { M } } , \xi\in \T_X\cM.
\label{apd:Retrbound2}
\end{align}
\end{subequations}
\end{lemma}

\begin{lemma}\label{lem:lipschitz_h_c}
The function $h$ in \eqref{SSC-manopt-rewrite} is $L_h$-Lipschitz continuous and the Jacobian $\nabla c(x)$ is $L_c$-Lipschitz continuous, where $L_h=n\lambda$ and $L_c=2$. 
\end{lemma}
\begin{proof}
	Since $h(Z)=\lambda\|Z\|_1$ where $Z\in\br^{n\times n}$, we immediately get that $h$ is $n\lambda$-Lipschitz continuous. Secondly, observing that $\nabla c(U)V = UV^\top + VU^\top $ for any $V\in\br^{n\times k}$, we have 
	\[ \|\nabla c(U_1)  - \nabla c(U_2) \|_{\text{op}} =\max_{\normfro{V}=1 } \|(U_1-U_2)V^\top +V(U_1-U_2)^\top \|_{F}\leq 2\normfro{U_1-U_2}.    \]
	Hence,  $\nabla c(x)$ is $2$-Lipschitz continuous. 
\end{proof}

\section{Proofs}
 
\subsection{Convergence Analysis of ManPL (Algorithm \ref{alg:manpl-ssc})}\label{sec:proof manpl} 
Before we present the convergence result of ManPL, we need the following useful lemma for proximal linear algorithm.
\begin{lemma}\citep{drusvyatskiy2018efficiency} 
For functions $h$ and $c$ satisfying the properties in Lemma \ref{lem:lipschitz_h_c}, we have the following result:
\begin{equation}
-\frac{L_cL_h}{2}\|U_2-U_1\|_F^2\leq h\Bigl(c(U_2)\Bigr) - h\Bigl(c(U_1) + \nabla c(U_1)(U_2-U_1)\Bigr)\leq \frac{L_cL_h}{2}\|U_2-U_1\|_F^2.
\label{apd:prox_weaU^{k+1}}
\end{equation}
\label{apd:lemma1}
\end{lemma}

\begin{proof}
It follows form Lemma \ref{lem:lipschitz_h_c} that
\begin{align*}
& \Bigl|h(c(U_2)) - h(c(U_1) + \nabla c(U_1)(U_2-U_1))\Bigr| \\
\leq & L_h\|c(U_2)- \left(c(U_1)+\nabla c(U_1)(U_2-U_1)\right)\|_F\\
=&L_h\left\|\int_{0}^{1}\left(\nabla c(U_1+t(U_2-U_1)) - \nabla c(U_1)\right)(U_2-U_1)dt\right\|_F
\\ \leq & L_h\int_{0}^{1}\|\left(\nabla c(U_1+t(U_2-U_1)) -\nabla c(U_1)\right)\|_{\text{op}}\|U_2-U_1\|_Fdt \\
\leq & L_hL_c\left(\int_0^1 tdt\right)\|U_2-U_1\|_F^2 = \frac{L_cL_h}{2}\|U_1-U_2\|_F^2.
\end{align*} 
This completes the proof. 
\end{proof}

Now, we start to analyze the convergence and iteration complexity of our ManPL algorithm. 
The  following lemma states that the minimizer of \eqref{ManPL-alg-1} is a descent direction on the tangent space. 
\begin{lemma}\label{apd:usefullemma}
Let $V^{k}$ be the minimizer of \eqref{ManPL-alg-1}, the following holds for any $\alpha \in [0,1]$, with any $t>0$: 
\begin{equation}\label{apd:Useful_equation}
\begin{aligned}
\langle \nabla f(U^{k})),\alpha V^{k}\rangle + \frac{1}{2t}\normfro{\alpha V^{k}}^2+ h\Bigl(c(U^{k}) + \alpha\nabla c(U^{k})V^{k}\Bigr) - h\Bigl(c(U^{k})\Bigr)\leq  \frac{\alpha^2-2\alpha}{2t}\|V^{k}\|_F^2.
\end{aligned}
\end{equation}
\end{lemma}

\begin{proof}
Since $V^{k}$ is the minimizer of \eqref{ManPL-alg-1}, for any $\alpha \in [0,1)$, we have:
\begin{align*}
    & \langle \nabla f(U^{k}),\alpha V^{k}\rangle + \frac{1}{2t}\|\alpha V^{k}\|_F^2 + h\Bigl(c(U^{k}) + \alpha \nabla c (U^{k})V^{k}\Bigr)\\
    \geq &  \langle \nabla f(U^{k}),V^{k}\rangle + \frac{1}{2t}\|V^{k}\|_F^2 + h\Bigl(c(U^{k}) + \nabla c(U^{k})V^{k}\Bigr),
\end{align*}
which implies that 
\begin{equation*}
(1 - \alpha)\langle \nabla f(U^{k}),V^{k}\rangle + \frac{1 - \alpha^2}{2t}\|V^{k}\|_F^2+h\Bigl(c(U^{k}) + \nabla c(U^{k})V^{k}\Bigr)
 - h\Bigl(c(U^{k}) + \alpha \nabla c(U^{k})V^{k}\Bigr) \leq 0.
\end{equation*}
Using the convexity of $h$, we have 
\begin{equation*}
\langle \nabla f(U^{k}),V^{k}\rangle + \frac{1+\alpha}{2t}\|V^{k}\|_F^2 + h\Bigl(c(U^{k}) + \nabla c(U^{k})V^{k}\Bigr) - h\Bigl(c(U^{k})\Bigr)\leq 0.
\end{equation*}
Letting $\alpha \rightarrow 1$, we get 
\begin{equation*}
\langle \nabla f(U^{k})),V^{k}\rangle + h\Bigl(c(U^{k}) + \nabla c(U^{k})V^{k}\Bigr) - h\Bigl(c(U^{k})\Bigr) \leq  -\frac{1}{t}\|V^{k}\|_F^2.
\end{equation*}
Finally, from the convexity of $h$ we get 
\begin{equation*}
\begin{aligned}
&\langle \nabla f(U^{k})),\alpha V^{k}\rangle + \frac{1}{2t}\normfro{\alpha V^{k}}^2+ h\Bigl(c(U^{k}) + \alpha\nabla c(U^{k})V^{k}\Bigr) - h\Bigl(c(U^{k})\Bigr) \\
\leq & \alpha\left(\langle  \nabla f(U^{k}),  V^{k}\rangle  +  h\Bigl(c(U^{k}) + \nabla c(U^{k})V^{k}\Bigr) - h\Bigl(c(U^{k})\Bigr)\right) + \frac{\alpha^2}{2t }\|V^{k}\|_F^2
\\
\leq &  \frac{\alpha^2-2\alpha}{2t}\|V^{k}\|_F^2,
\end{aligned}
\end{equation*}
which completes the proof. 
\end{proof}

The following lemma suggests that the new point ${U^{k+1}} =  \Retr_{U^{k}}(\alpha V^{k})$ has a lower function value. 

\begin{lemma}\label{lemma:ManPL_convergence}
Given any $t>0$, for $V^{k}$ and \({U^{k+1}}\) computed in \eqref{ManPL-alg}, there exists a constant $\bar{\alpha} > 0$ such that
\begin{equation}\label{apd:ManPL_convergence}
    F( {U^{k+1}}) - F( {U^{k}}) \leq -\frac{\alpha}{2t}\|V^{k}\|_F^2, \quad  \forall\  0\leq \alpha\leq \min\{1,\bar{\alpha}\}.
\end{equation}
\end{lemma}

\begin{proof}
We will prove by induction.  Denote \(U^{k}_+ = U^{k} + \alpha V^{k}\).   For $k = 0$, 
by using the convexity of $h$ and Lipchitz continuity of $c$, we can show that:
\begin{align} \label{apd:result_1}
\begin{split}
      & h\Bigl(c({U^{k+1}})\Bigr) -h\Bigl(c(U^{k}) + \alpha\nabla c(U^{k})V^{k}\Bigr)\\
    = &h\Bigl(c({U^{k+1}})\Bigr) - h\Bigl(c(U^{k}) + \nabla c(U^{k})( {U^{k+1}} - U^{k})\Bigr)\\
      &+h\Bigl(c(U^{k}) + \nabla c(U^{k})({U^{k+1}} - U^{k})\Bigr) -h\Bigl(c(U^{k}) +  \nabla c(U^{k})(U^{k}_+ - U^{k})\Bigr) \\
    \stackrel{\eqref{apd:prox_weaU^{k+1}}}{\leq} &\frac{  L_hL_c}{2} \|{U^{k+1}}-U^{k}\|_F^2 + L_h\|\nabla c(U^{k})({U^{k+1}} -U^{k}_+)\|_F\\
     \stackrel{\eqref{apd:Retrbound1}}{\leq}&\frac{\alpha^2 L_hL_c}{2}M_1^2\|V^{k}\|_F^2 + L_h\|\nabla c(U^{k})({U^{k+1}} -U^{k}_+)\|_F\\
    \stackrel{\eqref{apd:Retrbound2}}{\leq} &\frac{\alpha^2 L_cL_h}{2}M_1^2\|V^{k}\|_F^2 + L_h L_c M_2 \alpha^2\|V^{k}\|_F^2 = \Bigl(\frac{1}{2}M_1^2 + M_2\Bigr)L_cL_h\alpha^2\|V^{k}\|_F^2.
    \end{split}   
\end{align}
Since $\nabla f(X)$ is $L_f$ Lipschiz continuous, we  obtain
\begin{align}
\begin{split}
    f( {U^{k+1}}) - f(U^{k}) &\leq  \langle \nabla f(U^{k}), {U^{k+1} } - U^{k}\rangle + \frac{L_f}{2}\| {U^{k+1}}- U^{k}\|_F^2\\
    & = \langle \nabla f(U^{k}), {U^{k+1}}-U^{k}_+ +{U^{k}}_+ -U^{k}\rangle + \frac{L_f}{2}\| {U^{k+1}} - U^{k}\|_F^2\\
    & \stackrel{\eqref{apd:Retrbound2}}{\leq} M_2\|\nabla f(U^{k})\|_F\|\alpha V^{k}\|_2^2 + \alpha \langle \nabla f(U^{k}),V^{k}\rangle + \frac{M_1^2L_f}{2}\|\alpha V^{k}\|_F^2\\
    &  \leq c_0 \alpha^2\|V^{k}\|_F^2 + \alpha \langle \nabla f (U^{k}),V^{k}\rangle,
    \end{split}
    \label{apd:result_2}
\end{align}
where $c_0 = M_2G+\frac{M_1^2 L_f}{2}$ and $G=\max_{U\in\cM}\normfro{\nabla f(U)}$.

Now we need to show that the objective value decreases. 
It follows that
\begin{align*}
    & F({U^{k+1}}) -F({U^{k}})\\
\stackrel{\eqref{apd:result_2}}{\leq} & \alpha \langle \nabla f(U^{k}) ,V^{k}\rangle + c_0\alpha^2\|V^{k}\|^2 + h\Bigl(c({U^{k+1}})\Bigr) -h\Bigl(c({U^{k}})+\alpha\nabla c(U^{k}) V^{k} \Bigr) \\
& +h\Bigl(c({U^{k}})+\alpha\nabla c(U^{k}) V^{k} \Bigr) - h\Bigl(c({U^{k}})\Bigr)\\
\stackrel{\eqref{apd:result_1}}{\leq} & \alpha \langle \nabla f(U^{k}),V^{k}\rangle + c_0 \alpha^2\|V^{k}\|_2^2 + (\frac{1}{2}M_1^2 + M_2)L_cL_h\alpha^2\|V^{k}\|_2^2 \\
&+ \alpha \left(h\Bigl(c(U^{k}) + \nabla c(U^{k}) V^{k} \Bigr) - h\Bigl(c(U^{k})\Bigr)\right)\\
\stackrel{\eqref{apd:Useful_equation}
}{\leq} & \left[\left(c_0 + \left(\frac{1}{2}M_1^2 + M_2\right)L_cL_h \right)\alpha^2 - \frac{\alpha}{t} \right]\|V^{k}\|_F^2.
\end{align*}

 Letting  
\(\overline{\alpha} =\frac{1}{2\Bigl(c_0 + \Bigl(\frac{1}{2}M_1^2 + M_2\Bigr)L_cL_h \Bigr)t}\), we get that for any $0\leq \alpha \leq \min\{ 1,\bar{\alpha} \}$, 
\begin{equation}
    F(\Retr_{U^{k}}(\alpha V^{k})) - F({U^{k}}) \leq -\frac{\alpha}{2t}\|V^{k}\|_F^2.
\end{equation}
Thus, the result \eqref{apd:ManPL_convergence} holds for ${k}=0$. 
Using induction and the same procedure as above, we can show that \eqref{apd:ManPL_convergence} holds for all $k \geq 1$.
The proof is completed.
\end{proof}

Following \citep{chen2018proximal}, we now define the $\epsilon$-stationary point of \eqref{SSC-manopt-rewrite}.
\begin{definition}\label{def:eps-sp}
Given any $t>0$,	$U^{k}$ is called an $\epsilon$-stationary point of \eqref{SSC-manopt-rewrite}, if $V^{k}$ computed in \eqref{ManPL-alg-1} satisfies $\|V^{k}/t\|_F \leq \epsilon$.
\end{definition}

The following lemma indicates that Definition \ref{def:eps-sp} is a reasonable defintion for the $\epsilon$-stationary point.
\begin{lemma}\label{optimal_condition}
If $V^{k} = 0$, then $U^{k}$ is a stationary point of problem \eqref{SSC-manopt-rewrite}.
\end{lemma}

\begin{proof} 
The optimality conditions of \eqref{ManPL-alg-1} are given by:
\begin{equation*}
0 \in \Proj_{\T_{U^k}\M}{\nabla c(U^{k})}^\top\partial h(c(U^{k}) + \nabla{c(U^{k})} V^{k})+\frac{1}{t}V^{k} + \grad f(U^{k}).
\end{equation*}
If $V^{k} = 0$, it follows that 
\begin{equation*}
0 \in \Proj_{\T_{U^k}\M}{\nabla c(U^{k})}^\top\partial h(c(U^{k})) + \grad f(U^{k}),
\end{equation*}
which is the first-order optimality condition of \eqref{SSC-manopt-rewrite}. 
\end{proof}

\begin{proof}{Proof of Theorem \ref{thm:manpl-convergence}.}
	By using Lemma \ref{lemma:ManPL_convergence}, the claimed result directly follows from the proof of \citep[Theorem 5.5]{chen2018proximal}. We omit the details for brevity. 
\end{proof}

\subsection{Convergence Analysis of AManPL (Algorithm \ref{alg:MKSSC})}\label{sec:proof of theorem amanpl}
 
\begin{lemma}(Sufficient decrease in the $w$ subproblem)\label{lem:suff_decrease_w} 
From \eqref{prob:mkssc-w}, we have
\be\label{lem:suff_decrease_w-eq}
\bar F({U^{k+1}},w^{k+1}) - \bar F({U^{k+1}},w^{k}) \leq  -\frac{\rho}{2 }\| w^{k+1} - w^k\|_1^2.
\ee
\end{lemma}

\begin{proof}
This follows from the $\rho$-strong convexity of the objective function of \eqref{prob:mkssc-w} over the probability simplex. 
\end{proof}

\begin{lemma}\label{lem:sufficient_descent}
 Let $(V^k,U^{k+1})$ be computed from \eqref{AManPL-alg-U}. 
There exists $\hat{\alpha}>0$ such that  for any $0\leq \alpha\leq \min\{1,\hat{\alpha}\}$, it holds that 
\begin{equation}\label{apd:convergence_result}
\bar F({U^{k+1}}, w^{k+1}) -  \bar F({U^{k}},w^{k}) \leq -\left(\frac{\alpha}{2t}  \|V^{k}\|_F^2 + \frac{\rho}{2}\|w^{k+1}-w^{k}\|_1^2\right).
\end{equation}
\end{lemma}

\begin{proof}
 Firstly we  show that the subproblem \eqref{prob:mkssc-U} is of the same form as \eqref{SSC-manopt-rewrite}. That is, we show that 
	$f(U,w^{k}) = \sum_{\ell=1}^T w_\ell^{k} \langle L^{(\ell)},UU^{\top}\rangle$ is smooth with respect $U$ and its gradient with respect $U$ is Lipschitz continuous. Once we prove this, then \eqref{apd:convergence_result} follows by combining Lemma \ref{lem:suff_decrease_w} and the proof of Lemma \ref{lemma:ManPL_convergence}. Note that for any $w\in \mathbb{R}^T$, we have: 
\begin{align*}
\|\nabla_{U} f(U_1,w^{k}) - \nabla_{U} f(U_2,w^{k})\|_F & = \normfro{ 2 \sum_{\ell=1}^T w_\ell^{k}  L^{(\ell)} (U_1-U_2)} \\ & \leq 2\sum_{\ell=1}^T  w_\ell^{k}\|L^{(\ell)}\|_2\|U_1 -U_2\|_F \leq  2 \|U_1 - U_2\|_F,
\end{align*}
where we use the fact that $\|L^{(\ell)}\|_2 = 1$ for every $\ell$ (followed by the property of the affinity matrix). This implies that for fixed $w^{k}$, $f(\cdot,w^k)$ is smooth and its gradient is Lipschitz continuous with Lipschitz constant $L_{f} = 2$. Now we can use Lemma \ref{lemma:ManPL_convergence}
and let 
\(\hat{\alpha} =\frac{1}{2\Bigl(c_1 + \Bigl(\frac{1}{2}M_1^2 + M_2\Bigr)L_cL_h \Bigr)t}\), where $c_1 = M_2G_1+\frac{L_{f} M_1^2  }{2}$ and 
\[G_1=\max\limits_{U\in\cM, \sum_\ell w_\ell =1, 0\leq w\leq 1}\normfro{\nabla_U f(U,w)},\]
$M_1$ and $M_2$ are constants defined in Lemma \ref{lem:two-inequalities}. From Lemma \ref{lemma:ManPL_convergence} we get that for any $0\leq \alpha \leq \min\{ 1,\hat{\alpha} \}$, 
\begin{equation*}
\bar F(\Retr_{U^{k}}(\alpha V^{k}),w^k) - \bar F({U^{k}},w^k) \leq -\frac{\alpha}{2t}\|V^{k}\|_F^2,
\end{equation*}
which together with \eqref{lem:suff_decrease_w-eq} yields the desired result \eqref{apd:convergence_result}. 
\end{proof}

Since for the $w$-subproblem \eqref{prob:mkssc-park-w}, one always has \eqref{ineq:w_opt}. 
This movtivates us to define the following $\epsilon$-stationary point for problem \eqref{prob:mkssc}. 
\begin{definition}\label{def:eps-AManPL}
	We call a point $(U^k,w^k)$ an $\epsilon$-stationary point of problem \eqref{prob:mkssc} if 
\[ \normfro{V^k/t}  \leq \epsilon,   \]
	where $V_k$ is the optimal solution to \eqref{AManPL-alg-U-1}. 
\end{definition}

The following lemma indicates that Definition \ref{def:eps-AManPL} is a reasonable definition for the $\epsilon$-stationary point of problem \eqref{prob:mkssc}.
\begin{lemma}
For given $(U^k,w^k)$ generated from the $k$-th iteration of AManPL, if $V^k=0$, 	where $V^k$ is the optimal solution to \eqref{AManPL-alg-U-1}, then $(U^k,w^k)$ is a stationary point of problem \eqref{prob:mkssc}.
\end{lemma}

\begin{proof}
Since $V^k = 0$, the optimality condition of \eqref{AManPL-alg-U-1} implies that 
\begin{align}\label{ineq:u_opt}
0 \in \Proj_{\T_{U^k}\M} \left( 2 \sum_{\ell=1}^T w_\ell^k L^{(\ell)} U^k + {\nabla c(U^k)}^\top\partial h(c(U^k))\right).
\end{align}
Moreover, the optimality condition to \eqref{prob:mkssc-w} implies that 
\begin{align}\label{ineq:w_opt}
0 \in c_\ell + \rho ( \log(w_\ell^k) +1) + \partial_{\ell} \mathbb{I}(w^k) ,
\end{align}
where $c_\ell = \langle U^{k}{U^{k}}^\top, L^{(\ell)} \rangle$, $\ell=1,\ldots,T$, and $\mathbb{I}(w)$ denotes the indicator function of the probability simplex constraint $\mathcal{S}:= \{w\mid \sum_{\ell=1}^T w_\ell =1, w_\ell\geq 0, \ell=1,\ldots,T \}$, i.e., $\mathbb{I}(w) = 0$ if $w\in\mathcal{S}$, and $\mathbb{I}(w) = +\infty$ otherwise.  
Moreover, since 
\begin{align*}
 \partial_R \left(h(c(U))  ,\mathbb{I}(w) \right)   
 =  \Proj_{\T_U\M}\left( \nabla c(U)^\top\partial h(c(U)) \right)\times \partial \mathbb{I}(w),
\end{align*}
the first-order optimal conditions of \eqref{prob:mkssc} are given by
\be\label{prob:mkssc-optcond}
\begin{split}
	0 &\in \Proj_{\T_U\M} \left(2 \sum_{\ell=1}^T w_\ell L^{(\ell)} U + {\nabla c(U )}^\top\partial h(c(U)) \right), \\
	0 & \in \langle L^{(\ell)},UU^\top \rangle + \rho (\log(w_\ell) + 1) + \partial_{\ell} \mathbb{I}(w), \quad \forall \ell. 
\end{split}
\ee
From \eqref{ineq:u_opt} and \eqref{ineq:w_opt} we know that $(U^k,w^k)$ satisfies the optimality condition \eqref{prob:mkssc-optcond}. Therefore $(U^k,w^k)$ is a stationary point of \eqref{prob:mkssc}. 
\end{proof}

\begin{proof}{Proof of Theorem \ref{thm:amanpl-convergence}.}
Since the feasible region of \eqref{prob:mkssc} is bounded, there exists a convergent subsequence of $\{(U^k,w^k)\}$. We denote a limit point of $\{(U^k,w^k)\}$ as $\{(U^*,w^*)\}$. Lemma \ref{lem:sufficient_descent} indicates that $\{\bar{F}(U^k,w^k)\}$ monotonically decreases, and therefore we have
\[ \lim_{k\rightarrow \infty} \normfro{V^k/t}^2 + \|{\rho(w^{k+1} - w^k)}\|_1^2 = 0.  \] 
Note that the function $ h(c(U) + \nabla{c(U)} V )$ is convex with respect to $V$, therefore, taking limit for  \eqref{ineq:u_opt} and \eqref{ineq:w_opt} gives 
\begin{align*} 
0 \in \Proj_{\T_{U^*}\M} \left( 2 \sum_{\ell=1}^T w_\ell^* L^{(\ell)} U^* + {\nabla c(U^{*})}^\top\partial h(c(U^{*}))\right),   
\end{align*}
and 
\begin{align*} 
0 \in c_\ell^{*} + \rho ( \log(w_\ell^{*}) +1 ) + \partial_{\ell} \mathbb{I}(w^{*}) , \quad \ell = 1,\ldots,T. 
\end{align*}
This suggests that $(U^*, w^*)$ is a stationary point.  

Morevoer, given $N\geq  \frac{ 2t\hat{\alpha}\gamma (\bar F(U^0,w^0) - \bar F^*)}{\epsilon^2} $ it follows from Lemma \ref{lem:sufficient_descent} that 
\[ \min_{k=1,\ldots,N} \normfro{V^k/t}^2   \leq \epsilon^2. \]
The proof is completed. 
\end{proof}

\bibliographystyle{plainnat}
\bibliography{MKSSC_reference}

\begin{thebibliography}{52}
\providecommand{\natexlab}[1]{#1}
\providecommand{\url}[1]{\texttt{#1}}
\expandafter\ifx\csname urlstyle\endcsname\relax
  \providecommand{\doi}[1]{doi: #1}\else
  \providecommand{\doi}{doi: \begingroup \urlstyle{rm}\Url}\fi

\bibitem[Absil et~al.(2008)Absil, Mahony, and Sepulchre]{AbsMahSep2008}
P.-A. Absil, R.~Mahony, and R.~Sepulchre.
\newblock \emph{Optimization Algorithms on Matrix Manifolds}.
\newblock Princeton University Press, Princeton, NJ, 2008.
\newblock ISBN 978-0-691-13298-3.

\bibitem[Bendory et~al.(2018)Bendory, Eldar, and
  Boumal]{Boumal-phase-retrieval-2018}
T.~Bendory, Y.~C. Eldar, and N.~Boumal.
\newblock Non-convex phase retrieval from {STFT} measurements.
\newblock \emph{IEEE Transactions on Information Theory}, 2018.

\bibitem[Boumal(2016)]{Boumal-phase-synchronization-2016}
N.~Boumal.
\newblock Nonconvex phase synchronization.
\newblock \emph{SIAM Journal on Optimization}, 26\penalty0 (4):\penalty0
  2355--2377, 2016.

\bibitem[Boumal and Absil(2011)]{boumal2011rtrmc}
Nicolas Boumal and Pierre-Antoine Absil.
\newblock {RTRMC}: A {R}iemannian trust-region method for low-rank matrix
  completion.
\newblock In \emph{Advances in neural information processing systems}, pages
  406--414, 2011.

\bibitem[Boumal et~al.(2018)Boumal, Absil, and Cartis]{boumal2018global}
Nicolas Boumal, Pierre-Antoine Absil, and Coralia Cartis.
\newblock Global rates of convergence for nonconvex optimization on manifolds.
\newblock \emph{IMA Journal of Numerical Analysis}, 39\penalty0 (1):\penalty0
  1--33, 2018.

\bibitem[Buettner et~al.(2015)Buettner, Natarajan, Casale, Proserpio,
  Scialdone, Theis, Teichmann, Marioni, and Stegle]{buettner2015computational}
Florian Buettner, Kedar~N Natarajan, F~Paolo Casale, Valentina Proserpio,
  Antonio Scialdone, Fabian~J Theis, Sarah~A Teichmann, John~C Marioni, and
  Oliver Stegle.
\newblock Computational analysis of cell-to-cell heterogeneity in single-cell
  {RNA}-sequencing data reveals hidden subpopulations of cells.
\newblock \emph{Nature Biotechnology}, 33\penalty0 (2):\penalty0 155, 2015.

\bibitem[Charisopoulos et~al.(2019{\natexlab{a}})Charisopoulos, Chen, Davis,
  Diaz, Ding, and Drusvyatskiy]{Charisopoulos-prox-linear-MC-2019}
V.~Charisopoulos, Y.~Chen, D.~Davis, M.~Diaz, L.~Ding, and D.~Drusvyatskiy.
\newblock Low-rank matrix recovery with composite optimization: Good
  conditioning and rapid convergence.
\newblock \emph{https://arxiv.org/pdf/1904.10020.pdf}, 2019{\natexlab{a}}.

\bibitem[Charisopoulos et~al.(2019{\natexlab{b}})Charisopoulos, Davis,
  D{\'\i}az, and Drusvyatskiy]{charisopoulos2019composite}
Vasileios Charisopoulos, Damek Davis, Mateo D{\'\i}az, and Dmitriy
  Drusvyatskiy.
\newblock Composite optimization for robust blind deconvolution.
\newblock \emph{arXiv preprint arXiv:1901.01624}, 2019{\natexlab{b}}.

\bibitem[Chen et~al.(2019)Chen, Deng, Ma, and So]{chen2019manifold-Asilomar}
S.~Chen, Z.~Deng, S.~Ma, and A.~M.-C. So.
\newblock Manifold proximal point algorithms for dual principal component
  pursuit and orthogonal dictionary learning.
\newblock In \emph{Proceedings of the 2019 Asilomar Conference on Signals,
  Systems, and Computers.}, 2019.

\bibitem[Chen et~al.(2020{\natexlab{a}})Chen, Deng, Ma, and
  So]{chen2019manifold}
S.~Chen, Z.~Deng, S.~Ma, and A.~M.-C. So.
\newblock Manifold proximal point algorithms for dual principal component
  pursuit and orthogonal dictionary learning.
\newblock \emph{arXiv:2005.02356}, 2020{\natexlab{a}}.

\bibitem[Chen et~al.(2020{\natexlab{b}})Chen, Ma, Xue, and
  Zou]{Chen-AManPG-2019}
S.~Chen, S.~Ma, L.~Xue, and H.~Zou.
\newblock An alternating manifold proximal gradient method for sparse principal
  component analysis and sparse canonical correlation analysis.
\newblock \emph{INFORMS Journal on Optimization}, 2\penalty0 (3):\penalty0
  192--208, 2020{\natexlab{b}}.

\bibitem[Chen et~al.(2020{\natexlab{c}})Chen, Ma, So, and
  Zhang]{chen2018proximal}
Shixiang Chen, Shiqian Ma, Anthony Man-Cho So, and Tong Zhang.
\newblock Proximal gradient method for nonsmooth optimization over the
  {S}tiefel manifold.
\newblock \emph{SIAM J. Optimization}, 30\penalty0 (1):\penalty0 210--239,
  2020{\natexlab{c}}.

\bibitem[Cherian and Sra(2017)]{Sra-Riemannian-dictionary-learning-2016}
A.~Cherian and S.~Sra.
\newblock {R}iemannian dictionary learning and sparse coding for positive
  definite matrices.
\newblock \emph{IEEE Transactions on Neural Networks and Learning Systems},
  28\penalty0 (12):\penalty0 2859--2871, 2017.

\bibitem[Chung and Graham(1997)]{chung1997spectral}
Fan~RK Chung and Fan~Chung Graham.
\newblock \emph{Spectral Graph Theory}.
\newblock American Mathematical Society, 1997.

\bibitem[Deng et~al.(2014)Deng, Ramsk{\"o}ld, Reinius, and
  Sandberg]{deng2014single}
Qiaolin Deng, Daniel Ramsk{\"o}ld, Bj{\"o}rn Reinius, and Rickard Sandberg.
\newblock Single-cell {RNA}-seq reveals dynamic, random monoallelic gene
  expression in mammalian cells.
\newblock \emph{Science}, 343\penalty0 (6167):\penalty0 193--196, 2014.

\bibitem[Drusvyatskiy and Paquette(2018)]{drusvyatskiy2018efficiency}
Dmitriy Drusvyatskiy and Courtney Paquette.
\newblock Efficiency of minimizing compositions of convex functions and smooth
  maps.
\newblock \emph{Mathematical Programming}, pages 1--56, 2018.

\bibitem[Dua and Graff(2017)]{Dua:2019}
Dheeru Dua and Casey Graff.
\newblock {UCI} machine learning repository, 2017.
\newblock URL \url{http://archive.ics.uci.edu/ml}.

\bibitem[Duchi and Ruan(2018)]{Duchi-Ruan-SIOPT-2018}
J.~C. Duchi and F.~Ruan.
\newblock Stochastic methods for composite and weakly convex optimization
  problems.
\newblock \emph{SIAM Journal on Optimization}, 28\penalty0 (4):\penalty0
  3229--3259, 2018.

\bibitem[Duchi and Ruan(2017)]{duchi2017solving}
John~C Duchi and Feng Ruan.
\newblock Solving (most) of a set of quadratic equalities: Composite
  optimization for robust phase retrieval.
\newblock \emph{arXiv preprint arXiv:1705.02356}, 2017.

\bibitem[Ferreira and Oliveira(1998)]{ferreira1998subgradient}
O.~P. Ferreira and P.~R. Oliveira.
\newblock Subgradient algorithm on {R}iemannian manifolds.
\newblock \emph{Journal of Optimization Theory and Applications}, 97\penalty0
  (1):\penalty0 93--104, 1998.

\bibitem[Friedman et~al.(2001)Friedman, Hastie, and
  Tibshirani]{friedman2001elements}
Jerome Friedman, Trevor Hastie, and Robert Tibshirani.
\newblock \emph{The Elements of Statistical Learning}.
\newblock Springer Series in Statistics New York, 2001.

\bibitem[Grohs and Hosseini(2016)]{grohs2016varepsilon}
P.~Grohs and S.~Hosseini.
\newblock $\varepsilon$-subgradient algorithms for locally lipschitz functions
  on {R}iemannian manifolds.
\newblock \emph{Advances in Computational Mathematics}, 42\penalty0
  (2):\penalty0 333--360, 2016.

\bibitem[Hosseini and Pouryayevali(2011)]{hosseini2011generalized}
S~Hosseini and MR~Pouryayevali.
\newblock Generalized gradients and characterization of epi-lipschitz sets in
  {R}iemannian manifolds.
\newblock \emph{Nonlinear Analysis: Theory, Methods \& Applications},
  74\penalty0 (12):\penalty0 3884--3895, 2011.

\bibitem[Hosseini and Uschmajew(2017)]{Hosseini-Uschmajew-2017}
S.~Hosseini and A.~Uschmajew.
\newblock A {R}iemannian gradient sampling algorithm for nonsmooth optimization
  on manifolds.
\newblock \emph{SIAM Journal on Optimization}, 27\penalty0 (1):\penalty0
  173--189, 2017.

\bibitem[Huang and Wei(2019{\natexlab{a}})]{HW2019a}
Wen Huang and Ke~Wei.
\newblock Extending {FISTA} to {R}iemannian optimization for sparse {PCA}.
\newblock Technical report, 2019{\natexlab{a}}.

\bibitem[Huang and Wei(2019{\natexlab{b}})]{HW2019b}
Wen Huang and Ke~Wei.
\newblock Riemannian proximal gradient methods.
\newblock Technical report, 2019{\natexlab{b}}.

\bibitem[Kiselev et~al.(2019)Kiselev, Andrews, and
  Hemberg]{kiselev2019challenges}
Vladimir~Yu Kiselev, Tallulah~S Andrews, and Martin Hemberg.
\newblock Challenges in unsupervised clustering of single-cell {RNA}-seq data.
\newblock \emph{Nature Reviews Genetics}, page~1, 2019.

\bibitem[Lewis and Wright(2016)]{lewis2016proximal}
Adrian~S Lewis and Stephen~J Wright.
\newblock A proximal method for composite minimization.
\newblock \emph{Mathematical Programming}, 158\penalty0 (1-2):\penalty0
  501--546, 2016.

\bibitem[Li et~al.(2018)Li, Sun, and Toh]{Li-Sun-Toh-Lasso}
X.~Li, D.~Sun, and K.-C. Toh.
\newblock A highly efficient semismooth {N}ewton augmented {L}agrangian method
  for solving {L}asso problems.
\newblock \emph{SIAM J. Optimization}, 28\penalty0 (1):\penalty0 433--458,
  2018.

\bibitem[Li et~al.(2019)Li, Chen, Deng, Qu, Zhu, and So]{li2019nonsmooth}
X.~Li, S.~Chen, Z.~Deng, Q.~Qu, Z.~Zhu, and A.~M.-C. So.
\newblock Weakly convex optimization over {S}tiefel manifold using {R}iemannian
  subgradient-type methods.
\newblock \emph{https://arxiv.org/abs/1911.05047}, 2019.

\bibitem[Liu et~al.(2017)Liu, Yue, and
  So]{Liu-generalized-power-phase-synchronization-2017}
H.~Liu, M.-C. Yue, and A.~M.-C. So.
\newblock On the estimation performance and convergence rate of the generalized
  power method for phase synchronization.
\newblock \emph{SIAM Journal on Optimization}, 27\penalty0 (4):\penalty0
  2426--2446, 2017.

\bibitem[Lu et~al.(2016)Lu, Yan, and Lin]{lu2016convex}
Canyi Lu, Shuicheng Yan, and Zhouchen Lin.
\newblock Convex sparse spectral clustering: Single-view to multi-view.
\newblock \emph{IEEE Transactions on Image Processing}, 25\penalty0
  (6):\penalty0 2833--2843, 2016.

\bibitem[Lu et~al.(2018)Lu, Feng, Lin, and Yan]{lu2018nonconvex}
Canyi Lu, Jiashi Feng, Zhouchen Lin, and Shuicheng Yan.
\newblock Nonconvex sparse spectral clustering by alternating direction method
  of multipliers and its convergence analysis.
\newblock In \emph{Thirty-Second AAAI Conference on Artificial Intelligence},
  2018.

\bibitem[Mifflin(1977)]{mifflin1977semismooth}
Robert Mifflin.
\newblock Semismooth and semiconvex functions in constrained optimization.
\newblock \emph{SIAM Journal on Control and Optimization}, 15\penalty0
  (6):\penalty0 959--972, 1977.

\bibitem[Ng et~al.(2002)Ng, Jordan, and Weiss]{ng2002spectral}
Andrew~Y Ng, Michael~I Jordan, and Yair Weiss.
\newblock On spectral clustering: Analysis and an algorithm.
\newblock In \emph{Advances in Neural Information Processing Systems}, pages
  849--856, 2002.

\bibitem[Park and Zhao(2018)]{park2018spectral}
Seyoung Park and Hongyu Zhao.
\newblock Spectral clustering based on learning similarity matrix.
\newblock \emph{Bioinformatics}, 34\penalty0 (12):\penalty0 2069--2076, 2018.

\bibitem[Pollen et~al.(2014)Pollen, Nowakowski, Shuga, Wang, Leyrat, Lui, Li,
  Szpankowski, Fowler, Chen, et~al.]{pollen2014low}
Alex~A Pollen, Tomasz~J Nowakowski, Joe Shuga, Xiaohui Wang, Anne~A Leyrat,
  Jan~H Lui, Nianzhen Li, Lukasz Szpankowski, Brian Fowler, Peilin Chen, et~al.
\newblock Low-coverage single-cell {mRNA} sequencing reveals cellular
  heterogeneity and activated signaling pathways in developing cerebral cortex.
\newblock \emph{Nature Biotechnology}, 32\penalty0 (10):\penalty0 1053, 2014.

\bibitem[Qi and Sun(1993)]{qi1993nonsmooth}
Liqun Qi and Jie Sun.
\newblock A nonsmooth version of {N}ewton's method.
\newblock \emph{Mathematical Programming}, 58\penalty0 (1-3):\penalty0
  353--367, 1993.

\bibitem[Schlitzer et~al.(2015)Schlitzer, Sivakamasundari, Chen, Sumatoh,
  Schreuder, Lum, Malleret, Zhang, Larbi, Zolezzi,
  et~al.]{schlitzer2015identification}
Andreas Schlitzer, V~Sivakamasundari, Jinmiao Chen, Hermi Rizal~Bin Sumatoh,
  Jaring Schreuder, Josephine Lum, Benoit Malleret, Sanqian Zhang, Anis Larbi,
  Francesca Zolezzi, et~al.
\newblock Identification of c{DC}1-and c{DC}2-committed {DC} progenitors
  reveals early lineage priming at the common {DC} progenitor stage in the bone
  marrow.
\newblock \emph{Nature Immunology}, 16\penalty0 (7):\penalty0 718, 2015.

\bibitem[Shi and Malik(2000)]{shi2000normalized}
Jianbo Shi and Jitendra Malik.
\newblock Normalized cuts and image segmentation.
\newblock \emph{IEEE Transactions on Pattern Analysis and Machine
  Intelligence}, 22\penalty0 (8):\penalty0 888--905, 2000.

\bibitem[Sun et~al.(2017)Sun, Qu, and Wright]{Sun-CDR-part1-2017}
J.~Sun, Q.~Qu, and J.~Wright.
\newblock Complete dictionary recovery over the sphere {I}: Overview and the
  geometric picture.
\newblock \emph{IEEE Transactions on Information Theory}, 63\penalty0
  (2):\penalty0 853--884, 2017.

\bibitem[Sun et~al.(2018)Sun, Qu, and
  Wright]{Sun-Ju-geometric-phase-retrieval-2018}
J.~Sun, Q.~Qu, and J.~Wright.
\newblock A geometrical analysis of phase retrieval.
\newblock \emph{Foundations of Computational Mathematics}, 18\penalty0
  (5):\penalty0 1131--1198, 2018.

\bibitem[Ting et~al.(2014)Ting, Wittner, Ligorio, Jordan, Shah, Miyamoto,
  Aceto, Bersani, Brannigan, Xega, et~al.]{ting2014single}
David~T Ting, Ben~S Wittner, Matteo Ligorio, Nicole~Vincent Jordan, Ajay~M
  Shah, David~T Miyamoto, Nicola Aceto, Francesca Bersani, Brian~W Brannigan,
  Kristina Xega, et~al.
\newblock Single-cell {RNA} sequencing identifies extracellular matrix gene
  expression by pancreatic circulating tumor cells.
\newblock \emph{Cell Reports}, 8\penalty0 (6):\penalty0 1905--1918, 2014.

\bibitem[Treutlein et~al.(2014)Treutlein, Brownfield, Wu, Neff, Mantalas,
  Espinoza, Desai, Krasnow, and Quake]{treutlein2014reconstructing}
Barbara Treutlein, Doug~G Brownfield, Angela~R Wu, Norma~F Neff, Gary~L
  Mantalas, F~Hernan Espinoza, Tushar~J Desai, Mark~A Krasnow, and Stephen~R
  Quake.
\newblock Reconstructing lineage hierarchies of the distal lung epithelium
  using single-cell {RNA}-seq.
\newblock \emph{Nature}, 509\penalty0 (7500):\penalty0 371, 2014.

\bibitem[van~der Maaten and Hinton(2008)]{maaten2008visualizing}
Laurens van~der Maaten and Geoffrey Hinton.
\newblock Visualizing data using t-{SNE}.
\newblock \emph{Journal of Machine Learning Research}, 9\penalty0
  (Nov):\penalty0 2579--2605, 2008.

\bibitem[Wang et~al.(2020)Wang, Ma, and Xue]{Wang-stochastic-ManPG}
B.~Wang, S.~Ma, and L.~Xue.
\newblock Riemannian stochastic proximal gradient methods for nonsmooth
  optimization over the {S}tiefel manifold.
\newblock \emph{https://arxiv.org/pdf/2005.01209.pdf}, 2020.

\bibitem[Wang et~al.(2017)Wang, Ramazzotti, De~Sano, Zhu, Pierson, and
  Batzoglou]{wang2017simlr}
Bo~Wang, Daniele Ramazzotti, Luca De~Sano, Junjie Zhu, Emma Pierson, and
  Serafim Batzoglou.
\newblock {SIMLR}: a tool for large-scale single-cell analysis by multi-kernel
  learning.
\newblock \emph{bioRxiv}, page 118901, 2017.

\bibitem[Wang et~al.(2010)Wang, Sun, and Toh]{wang2010solving}
Chengjing Wang, Defeng Sun, and Kim-Chuan Toh.
\newblock Solving log-determinant optimization problems by a {Newton-CG} primal
  proximal point algorithm.
\newblock \emph{SIAM Journal on Optimization}, 20\penalty0 (6):\penalty0
  2994--3013, 2010.

\bibitem[Xiao et~al.(2018)Xiao, Li, Wen, and Zhang]{xiao2018regularized}
Xiantao Xiao, Yongfeng Li, Zaiwen Wen, and Liwei Zhang.
\newblock A regularized semi-smooth {N}ewton method with projection steps for
  composite convex programs.
\newblock \emph{Journal of Scientific Computing}, pages 1--26, 2018.

\bibitem[Yang et~al.(2013)Yang, Sun, and Toh]{yang2013proximal}
Junfeng Yang, Defeng Sun, and Kim-Chuan Toh.
\newblock A proximal point algorithm for log-determinant optimization with
  group {L}asso regularization.
\newblock \emph{SIAM Journal on Optimization}, 23\penalty0 (2):\penalty0
  857--893, 2013.

\bibitem[Yang et~al.(2014)Yang, Zhang, and Song]{yang2014optimality}
Wei~Hong Yang, Lei-Hong Zhang, and Ruyi Song.
\newblock Optimality conditions for the nonlinear programming problems on
  {R}iemannian manifolds.
\newblock \emph{Pacific Journal of Optimization}, 10\penalty0 (2):\penalty0
  415--434, 2014.

\bibitem[Zhao et~al.(2010)Zhao, Sun, and Toh]{zhao2010newton}
Xin-Yuan Zhao, Defeng Sun, and Kim-Chuan Toh.
\newblock A {Newton-CG} augmented {L}agrangian method for semidefinite
  programming.
\newblock \emph{SIAM Journal on Optimization}, 20\penalty0 (4):\penalty0
  1737--1765, 2010.

\end{thebibliography}

\end{document}